\documentclass[11pt, a4paper]{article}

\usepackage{amsmath}
\usepackage{amssymb}
\usepackage{amsthm}
\usepackage{mathtools}
\usepackage{colortbl}
\usepackage{graphicx}
\graphicspath{{plots/}}
\usepackage{hyperref}
\usepackage{geometry}
\usepackage{algorithm}
\usepackage{algpseudocode}
\usepackage{booktabs}
\usepackage{caption}
\usepackage{subcaption}
\usepackage{array}
\usepackage{physics}
\usepackage{enumitem}
\usepackage[svgnames]{xcolor}
\usepackage{bm}
\usepackage{mathrsfs}

\usepackage{newtxtext}
\usepackage{newtxmath}
\usepackage[backend=biber,style=numeric,sorting=none,natbib=true,maxcitenames=1]{biblatex}

\usepackage{amsthm}
\usepackage{thmtools}
\usepackage{thm-restate}

\theoremstyle{plain}
\newtheorem{theorem}{Theorem}[section]
\newtheorem{proposition}[theorem]{Proposition}
\newtheorem{lemma}[theorem]{Lemma}
\newtheorem{corollary}[theorem]{Corollary}
\theoremstyle{definition}

\theoremstyle{remark}
\newtheorem{remark}[theorem]{Remark}

\usepackage{tikz}
\usepackage{graphicx}
\usetikzlibrary{
    calc,
    positioning,
    arrows.meta
}

\usepackage{cleveref}
\crefname{assumption}{Assumption}{Assumptions}
\crefname{equation}{Equation}{Equations}
\crefname{theorem}{Theorem}{Theorems}
\crefname{lemma}{Lemma}{Lemmas}
\crefname{proposition}{Proposition}{Propositions}
\crefname{corollary}{Corollary}{Corollaries}
\crefname{definition}{Definition}{Definitions}
\crefname{note}{Note}{Notes}
\crefname{table}{Table}{Tables}
\crefname{figure}{Figure}{Figures}
\crefname{example}{Example}{Examples}
\crefname{section}{Section}{Sections}
\crefname{app}{Appendix}{Appendices}

\makeatletter
\newcommand{\qedheresafe}{%
  \tag*{$\qedsymbol$}%
  \gdef\QED@stack{\qed@elt{}}%
}
\makeatother

\usepackage{pifont}

\let\Tr\relax
\let\dd\relax

\DeclarePairedDelimiterX{\infdivx}[2]{(}{)}{%
  #1\;\delimsize\|\;#2%
}

\newcommand{\Rb}{\mathbb{R}}

\newcommand{\Eb}{\mathbb{E}}

\newcommand{\bg}{\bar{g}}

\newcommand{\E}{\mathcal{E}}

\newcommand{\N}{\mathcal{N}}

\newcommand{\Hc}{\mathcal{H}}
\newcommand{\Lc}{\mathcal{L}}
\newcommand{\Ac}{\mathcal{A}}
\newcommand{\Sg}{\mathcal{S}_g}
\newcommand{\id}{\mathrm{Id}}
\newcommand{\e}{\mathcal{E}^\perp}
\newcommand{\GL}[1]{\mathrm{GL}({#1})}

\newcommand{\dd}{\mathrm{d}}

\DeclareMathOperator{\Tr}{Tr}

\usepackage{xparse}
\NewDocumentCommand{\inn}{O{x} O{x'} O{V_X}}{\left\langle #1, #2\right\rangle_{#3}}
\newcommand{\fs}{\mathsf{\Sigma}}
\newcommand{\fk}{\mathsf{K}}
\newcommand{\tix}{\Tilde{x}}
\newcommand{\tih}{\Tilde{h}}
\usepackage{multirow}

\usepackage{etoolbox} 
\newbool{includeapp}
\setbool{includeapp}{true} 

\usepackage{siunitx}

\hypersetup{
  colorlinks = true,
  allcolors = DarkRed,
}

\definecolor{color1}{HTML}{0077BB}
\definecolor{color2}{HTML}{EE7733}
\definecolor{color3}{HTML}{33BBEE}
\definecolor{color4}{HTML}{EE3377}
\definecolor{color5}{HTML}{CC3311}
\definecolor{color6}{HTML}{009988}
\definecolor{color7}{HTML}{BBBBBB}

\DeclareFieldFormat{urldate}{}

\DeclareSourcemap{
  \maps[datatype=bibtex]{
    \map{
      \step[fieldset=primaryclass, null]
    }
  }
}



\title{Boosting Data Augmentation with\\Stochastic Weight Averaging}

\makeatletter
\def\@fnsymbol#1{\ensuremath{\ifcase#1\or a\or b\or *\or c\or d\or e\or f\or g \or h \else\@ctrerr\fi}}
\makeatother

\author{%
  Longde Huang \footnotemark[1] \and
  Axel Flinth\footnotemark[3]{\ \,}\footnotemark[2] \and
  Jan E.\ Gerken\footnotemark[3]{\ \,}\footnotemark[1]
}

\date{}

\begin{document}
\renewcommand{\thefootnote}{\fnsymbol{footnote}}
\maketitle
\footnotetext[3]{Equal Contribution}
\footnotetext[1]{Department of Mathematical Sciences, Chalmers University of Technology and the University of Gothenburg, SE-412 96 Gothenburg, Sweden.\\ Emails: longde@chalmers.se, gerken@chalmers.se}
\footnotetext[2]{Department of Mathematics and Mathematical Statistics, Umeå University, Linnéus väg 49, 901 87 Umeå, Sweden.\\ Email: axel.flinth@umu.se}
\renewcommand{\thefootnote}{\arabic{footnote}}
\setcounter{footnote}{0}

\begin{abstract}
The symmetries of a learning task have become an important factor in designing modern deep learning solutions. Data augmentation is a straightforward and effective way of incorporating symmetries into a generic neural network. 
Recent results show that infinitely large deep ensembles show perfect symmetry when trained on augmented data. However, since training ensembles requires repeating the training process many times, this method is costly. In this work, we study stochastic weight averaging (SWA) applied to classification as an alternative ensembling technique that does not require repeated training runs. We analyze SWA by approximating the stochastic training trajectory at the end of training with an Ornstein--Uhlenbeck process. We show that in the infinite-width limit, SWA on augmented data provides an equivariance boost that goes beyond what could be expected from the performance increase due to SWA alone. We verify our results with extensive numerical experiments on numerous models spanning image and graph classification with both discrete and continuous symmetries.
\end{abstract}

\section{Introduction}
Many machine-learning tasks, ranging from protein folding~\citep{abramson2024} to medical image processing~\citep{bekkers2018} feature symmetries that constrain the learning problem and, if taken into account, can be used to improve the performance of the trained models. In deep learning more specifically, there are two major approaches for incorporating symmetries into the learning algorithm. 

Firstly, the symmetry can be realized as a constraint on the neural network architecture. In particular, the layers of the network can be chosen such that they are equivariant functions with respect to the symmetry of the task at hand~\citep{bronstein2021,gerken2023}. This approach requires specialized architectures (discussed in a large body of literature) but realizes the symmetry exactly and improves sample efficiency during training.

Secondly, the neural network can be trained on transformed versions of the original training data, a method known as data augmentation. Since this technique does not constrain the network architecture, it can be applied to any highly optimized off-the-shelf architecture as long as the symmetry transformation can be implemented. It is therefore straightforward and flexible. However, it does not lead to an exactly equivariant model, but rather one that is approximately equivariant (or equivariant in distribution).

Theoretical studies of equivariant neural networks are plentiful due to deep connections between equivariant layers and representation theory. In particular, practically relevant special cases such as the symmetry groups $S_n$ of permutations, $SO(n)$ of rotations or $SE(n)$ of roto-translations have been studied extensively in the literature, for various different model architectures.

In contrast, most studies of data augmentation are purely empirical, see e.g.~\cite{gerken2022}. This is because data augmentation modifies the training procedure and therefore one inevitably needs to study the complicated, non-linear training dynamics of neural networks. Despite this obstacle, recent advances in this area show that in expectation over the initialization distribution, neural networks are exactly equivariant when they are trained on augmented data~\citep{gerken2024,nordenfors2024ensembles}. Approximately, this can be realized by training a deep ensemble.

However, ensemble training is notoriously computationally intensive, since the training needs to be repeated for each ensemble member. For this reason, and motivated by the recent results mentioned above, we study the effects of data augmentation on \emph{stochastic weight averaging} (SWA). Instead of an ensemble over independent training trajectories, SWA takes an ensemble mean over several models from the same training trajectory~\citep{izmailov2019a}, thereby reducing the training cost to that of a single model. 

Intuitively, the training trajectory of SGD is dominated by noise from batch selection when it is close to a minimum of the loss. Therefore, SWA can be expected to have a similar effect to averaging over random initializations. In this work, we investigate whether this intuition holds up. In particular, we approximate the stochastic training trajectory at the end of training with an Ornstein--Uhlenbeck process and use the ratio of the expected losses with and without SWA to quantify the performance boost due to SWA.

A central obstacle to the theoretical analysis of SWA in this setting is the non-independence resulting from the samples being taken from a single training trajectory. 
To overcome this problem, we assume a well-trained network and express the performance boost in terms of the trace of the Hessian of the network. In this setting, we then consider the expected \emph{equivariance} boost which focuses on the non-equivariant contributions to the loss. In the infinite-width limit, we compute a bound on this expected equivariance boost using the theory of neural tangent kernels (NTKs). Our calculations show that combining data augmentation with SWA leads to an equivariance boost which goes beyond the mere performance increase expected from ordinary SWA. 


We validate our theoretical analysis with extensive numerical experiments. In particular, we explicitly verify that the bound predicted by our theory holds and has the right dependence on group size and averaging time on a synthetic problem over which we have full control. We furthermore demonstrate the claimed improvement on computer vision and graph classification tasks, as well as both discrete and continuous symmetry groups, for numerous different model classes.

Our main contributions are as follows:
\begin{itemize}
    \item We formalize SWA at the end of training with an Ornstein--Uhlenbeck process and analyse the equivariance properties of the resulting model in the case of augmented training data. To quantify the performance boost due to SWA, we introduce the loss ratio $R$ in \eqref{eq:r-loss}. Similarly, to quantify the equivariance boost due to SWA, we define the loss ratio $R_{\E^\perp}$ in \eqref{eq:r-perp}. We overcome the central obstacle of dependent ensemble members from one trajectory by assuming a well-trained network and expressing $R$ and $R_{\E^\perp}$ in terms of the trace of the Hessian of the network, see \eqref{eq:r-as-trH} and \eqref{eq:rperp-as-trH}.
    \item We then take the infinite-width limit and leverage the theory of NTKs to derive a bound on $R_{\E^\perp}$ (Theorem~\ref{thm:improvement-ratio-general}) which shows that SWA with data augmentation leads to a non-trivial improvement in equivariance for a wide range of models. In the process, we derive novel results about the NTK limit for equivariant models.
    \item We validate our theoretical results numerically by considering a synthetic learning task. We verify for this task that our bound still holds for a finitely-sized network and has the predicted dependence on the group size and averaging time (Section~\ref{sec:synth-expmts}). We furthermore test the equivariance gain from SWA in more realistic settings (Section~\ref{sec:practical}). In particular, we demonstrate the predicted improvement on five vision datasets with seven different models (MLPs, CNNs, vision and graph transformers) for discrete rotations. We also show that the same effect can be observed for a molecular graph classification task invariant with respect to 3D rotations.
\end{itemize}

\section{Related Work}
\textbf{Equivariance and data augmentation.} A large body of work builds equivariance directly into the network architecture, so that the symmetry constraint is imposed exactly throughout training rather than only approximately. Group-equivariant CNNs \citep{cohen2016groupequivariantconvolutionalnetworks} replaced standard convolutions with group convolutions over discrete groups, and were later generalized to continuous and steerable representations \citep{weiler2018group,cohen2018sphericalcnns}, ultimately unified under the framework of geometric deep learning \citep{bronstein2021}. Taking the reverse approach, \citet{kondor2018generalizationequivarianceconvolutionneural} studied equivariant maps from first principles, proving that convolutions are the unique equivariant linear maps for compact groups. These architectures guarantee equivariance without relying on the data distribution, at the cost of restricting the parameter space to a limited subspace, which \citet{nordenfors2025optimization} study in a general setting. This architectural restriction has been argued to improve generalization \citep{elesedy2021provablystrictgeneralisationbenefit,pmlr-v161-sannai21a}

From this parameter restriction arises the structural question: is layerwise equivariance the only way to realize equivariance? \citet{NEURIPS2023_2c74f005} first tackled this question in theory, arguing that this is not the general scenario despite the existence of supporting examples; they nevertheless conjectured that it holds empirically for trained CNNs. \citet{shahverdi2026identifiable}, under a stronger condition of identifiability, showed rigorously that end-to-end equivariance always implies the existence of a parameter choice realizing the same function with layerwise equivariance. This thread of works is the premise on which we treat the equivariant subspace $\mathcal E$ as well-defined, and study the training dynamics restricted to it.

Layerwise realizability does not make the architectural route costless in practice, since it requires committing to a specific group at design time, implementing specialized layers for it, and training within a restricted parameter space. As an alternative, data augmentation offers a route to equivariance that leaves the parameter space unconstrained, at the cost that the resulting network is only approximately equivariant. Whether this matters is contested. \citet{gerken2022} find that for invariant classification tasks for spherical images, sufficiently augmented models with large scale match equivariant architectures. Similarly, \cite{wang2024swallowingbitterpillsimplified} reports that dropping the rotational inductive bias simplifies model scaling considerably. \citet{brehmerDoesEquivarianceMatter2025}, in contrast, argues equivariant models outperform augmented ones at every compute budget tested. Augmentation thus appears as a viable substitute for architectural equivariance in some scenarios, including the classification tasks with finite groups considered in our work. \citet{chenGroupTheoreticFrameworkData2020a} give a theoretical framework for this alternative, modelling data augmentation as an averaging operation over group orbits that renders the data distribution approximately invariant, which is exactly the setting of our work.

\textbf{Ensembles and checkpoint averaging.} To achieve exact equivariance from data augmentation, ensembling has emerged as a provable mechanism \citep{gerken2024,nordenfors2024ensembles}. The result of \citet{nordenfors2024ensembles}, that an infinite ensemble of networks trained with fully augmented data is equivariant in expectation, appears in our more general framework as a special case. This connection to ensembling motivates a closer look at the broader literature on combining multiple models.

Ensembles of independently trained networks are a classical and extensively studied technique for improving generalization, and we refer to \citet{GANAIE2022105151} for a general review. Among the broad spectrum of methods, ensembling with multiple sequentially trained models has emerged as an efficient strategy. \citet{huangSnapshotEnsemblesTrain2017} and \citet{chenCheckpointEnsemblesEnsemble2017} constructed ensembles by averaging the outputs of models saved at different points during a single training round. Both works are based on the hypothesis that the training trajectory explores different local minima, and averaging over them reduces overfitting. \citet{garipovLossSurfacesMode2018}, on the other hand, explored the loss landscape further and discovered that local minima are connected by simple low-loss curves, challenging the isolated minima view. Inspired by this insight of mode connectivity, they traverse nearby low-loss regions to collect a diverse set of ensemble members.

A common issue among these methods is that they still require storing and evaluating multiple models at inference time, limiting their practicality. Weight averaging avoids this cost entirely by collapsing the ensemble into a single set of weights before inference, and is exactly the backbone of the mechanism we study in our work.

\textbf{Weight averaging.} Averaging the iterates of a stochastic optimization algorithm has classical roots in Polyak--Ruppert averaging \citep{article}. The most common form in deep learning is the exponential moving average (EMA) \citep{tarvainen2018meanteachersbetterrole}, a related but distinct technique that stochastic weight averaging (SWA) \citep{izmailov2019a} is often contrasted against. The two differ in their averaging scheme, which leads to different accounts of their effect on the loss landscape: by averaging iterates sampled uniformly rather than with exponential decay, SWA is argued to land in the interior of a wide, flat optimum, which is connected to better generalization. Several works adopt this wide-minima point of view. \citet{xieDiffusionTheoryDeep2021} ground this flat-minima intuition by modeling SGD as a diffusion process whose stationary distribution concentrates exponentially more on flat minima; \citet{damianLabelNoiseSGD2021} similarly show that SGD with label noise biases toward flatter minimizers by analyzing how the noise affects local quadratic structure of the loss.

This flat-minima account is not, however, unchallenged. \citet{guoStochasticWeightAveraging2022} revisit SWA empirically and find its benefit highly dependent on how well the underlying SGD trajectory has converged before averaging, arguing the operation contributes primarily through variance reduction rather than flat-minima discovery. This coincides with \citet{mandtStochasticGradientDescent2018}, who model SGD as a continuous-time SDE and show that near a well-fitted minimum, its stationary distribution is governed by a quadratic approximation of the loss with covariance set by the Hessian, which is the same picture we adopt. Unlike \citet{mandtStochasticGradientDescent2018}, however, who treat this as a general posterior characterization of SGD, we consider the setting where the dynamics decompose into equivariant and non-equivariant components, studying the impact of SWA on each in turn.

The quadratic, Hessian-governed picture is itself an active subject of empirical analysis. \citet{meterez2026defensequadraticmodel} stress-test it at LLM pretraining scale and find it remains predictive over a non-trivial fraction of training, while \citet{arousSpectralAlignmentStochastic2023} show that in high-dimensional classification tasks the SGD trajectory aligns with a low-dimensional outlier eigenspace of the Hessian, layer by layer. We likewise focus on classification tasks and study the spectrum of the Hessian, but apply the theory of the neural tangent kernel to derive a quantitative analysis.

\textbf{Neural tangent kernel.} The NTK \citep{jacot2018} characterizes the training dynamics of sufficiently wide neural networks as kernel gradient descent under a fixed, architecture-determined kernel, and is based on the network's correspondence to a Gaussian process at initialization known as the neural network Gaussian process (NNGP) \citep{lee2018deepneuralnetworksgaussian}. Subsequent work has used the spectrum of the NTK to characterize structure in the Hessian and gradients of trained networks. \citet{jacotAsymptoticSpectrumHessian2020} show that the NTK itself governs the asymptotic spectrum of the Hessian throughout training. \citet{basriFrequencyBiasNeural2020} use the NTK's spectral decomposition to characterize which frequencies a network learns fastest, extending the classical frequency-bias result to non-uniformly sampled data. \citet{seleznovaGradPCALeveragingNTK2025} build an OOD detector for classification models, based on the discovery of a related alignment phenomenon: the NTK's approximate block-diagonal structure across classes.

Closest to our setting in equivariant tasks, \citet{pmlr-v267-misof25a} compute the NTK of equivariant architectures whose intermediate representations are regular representations of a finite group. Our derivation in Section~\ref{sec:ntk_eq} generalizes this to arbitrary intermediate representations and irrep decompositions, and we use the resulting equivariant NTK to bound the trace of the equivariant Hessian $H_{\mathcal E}$ relative to the full Hessian $H$, which leads to our main theorem.

\section{Problem Setting}
In this section, we will set the notation and formalize stochastic weight averaging (SWA).

\textbf{Dataset and Data Augmentation.}
For a generic training task, we denote the true dataset as a joint distribution \(\mu_{X,Y}\) over \(V_X\times V_Y\), from which the training samples are drawn. Here, \(V_X= \Rb^{n_{X}}\) and \(V_Y= \Rb^{n_{Y}}\) are the input and output spaces. We assume the distribution \(\mu_{X,Y}\) to be supported on a compact set \(\Omega\). Furthermore, we assume a pair of group representations \((\rho_X,\,\rho_Y)\) of a group \(G\) on the input and output spaces, respectively,
\begin{equation}
\rho_X: G\rightarrow \GL{V_X}\,,\quad
\rho_Y: G\rightarrow \GL{{V_Y}}\,.
\end{equation}
In this context, we consider data augmentation of the training data as a condition on the distribution. Specifically, an augmented dataset corresponds to a distribution $\mu_{X,Y}$ that is invariant under the group action, namely the joint random variable \((X,\,Y)\) is identically distributed with \((\rho_X(g)X, \rho_Y(g)Y)\), \(\forall g\in G\). This also means \((\rho_X\times\rho_Y)(g)\Omega = \Omega,\,\forall g\in G\). For simplicity, we denote \(\rho_X(g)\), \(\rho_Y(g)\) with \(g_X, g_Y\)\label{notation:gxgy}.

\textbf{Network.}
As the neural network we train, we consider a multilayer perceptron (MLP) \(\N_\theta\) with parameters \(\theta=(W^0,\,W^1,\,\dots,\, W^{L})\) and nonlinearities \(\sigma = (\sigma_0, \dots, \sigma_L)\). In terms of the preactivations $(h^0,\dots,h^{L})$, the network function is given by 
\begin{align}
    h^0(x)=W^0x\,,
    \quad
    h^l(x)=\frac{1}{\sqrt{n_l}}W^{l}\sigma_{l-1}(h^{l-1}(x))
    \quad\text{for}\quad l=1,\dots,L\,,
    \quad
    \N_\theta(x)=h^{L}(x)\,,
    \label{eq:network}
\end{align}
where $h^l(x)\in\Rb^{n_{l+1}}$ and $W^{l}\in \Rb^{n_{l+1}\times n_l}$. In particular, we have \(n_X = n_0\) and \(n_Y = n_{L+1}\). We denote the space of parameters $\theta$ of this network by
\begin{align}
    \Hc=\bigoplus_{l=0}^L \Rb^{n_{l+1}\times n_l}\,.
\end{align}

To accommodate practical and modern architectures of interest, such as convolutional networks (CNN), transformers, and residual structures, we follow a standard parameter-space restriction and further restrict the admissible parameters \(\theta\) to an affine subspace \(\Lc\subset \Hc\). At layer $l$, we will denote space of admissible parameters by $\Lc^l$.

We call \(\N_\theta\) an equivariant network, if for \((x,y)\)-\(\mu_{X,Y}\) a.e. we have
\begin{equation}
    \N_\theta(g_Xx)=g_Y\N_\theta(x)\,.
\end{equation}

\textbf{Loss function.}
The loss function \(\ell(\N_{\theta}(x), y)\) is defined over $V_Y\times V_Y$. We assume it is non-negative, twice differentiable, locally Lipschitz and attains zero if and only if \(\N_\theta(x) = y\). We further assume that it is group invariant,
\begin{equation}
    \ell(g_Y\N_{\theta}(x), g_Yy)=\ell(\N_{\theta}(x), y)\,.\label{eq:inv_loss}
\end{equation}
Consequently, if the network \(\N_\theta\) itself is equivariant, then $\ell_\theta(x,y):=\ell(\N_{\theta}(x), y)$ is invariant,
\begin{equation}
    \ell_{\theta}(x,y) = \ell_{\theta}(g_Xx,g_Y y)\,.
\end{equation}
The cumulative loss, $L(\theta)$, is defined as the expectation of $\ell_\theta(x,y)$ with respect to the data distribution,
\begin{equation}
    L(\theta) = \Eb_\mu(\ell_\theta(x,y)) = \int_\Omega \ell_\theta(x,y)\dd \mu_{X,Y}\,.
\end{equation}

\textbf{Group action on the parameter space.}
For a network with a feed forward structure, one natural way to achieve equivariance is through layerwise equivariance. Therefore, we define a series of representations \(\rho_i\) on the intermediate spaces \(V_i=\Rb^{n_i}\).
We assume all representations, including the one on the input and the output space, are unitary with respect to the canonical tensor basis.

The group then naturally acts on the parameter space as
\begin{equation} \label{eq:natural_action}
    \bar{\rho}: G\rightarrow \GL{\Hc}\,,
    \quad
    \bar{\rho}(g)\theta=\bar{\rho}(g)(W^0,\dots, W^L)=(\rho_1(g)W^0\rho_X(g)^{-1},\dots, \rho_Y(g) W^{L} \rho_L(g)^{-1}).
\end{equation}
We require the non-linearities \(\sigma_i\) to be equivariant with respect to the corresponding intermediate representations \(\rho_i\),
\begin{equation} \label{eq:equiv_nonlin}
    \rho_i(g)\circ \sigma_i = \sigma_i\circ\rho_i(g),
\end{equation}
Note that it is equivalent that a set of parameters \(\theta\) is equivariant w.r.t. these representations,
\begin{equation}
    W^l\rho_{l}(g)=\rho_{l+1}(g)W^l\,,
\end{equation}
and that it is invariant under the associated representation \(\bar{\rho}\) (denote \(\bar{\rho}(g)\) as \(\bg\)),
\begin{equation}
    \theta = \bg \theta\,.
    \label{eq:equiv-theta}
\end{equation}
The parameters satisfying~\eqref{eq:equiv-theta} form a linear subspace \(\Hc_G\) in \(\Hc\). We assume the space of admissible parameters \(\Lc\) is invariant under the group action (\(\bg \Lc \subset \Lc,\,\forall g\in G\)), and the space of admissible equivariant parameters \(\E = \Hc_G\cap \Lc\) is non-empty. By picking \(\theta_G\in \E\), we can associate \(\Lc\) with \(T_{\theta_G}\Lc+\theta_G \). 

Assuming \eqref{eq:natural_action} and \eqref{eq:equiv_nonlin}, the group action on the data is equivalent to an action on the parameter space,
\begin{equation} \label{eq:parameter_transform}
    \N_{\theta}(g_Xx) =g_Y\N_{\bg^{-1}\theta}(x).
\end{equation}
This coincides with the definition of \(\E\), since if \(\theta \in \E\), the equation above turns into
\begin{equation}
    \N_{\theta}(g_Xx) =g_Y\N_{\theta}(x),
\end{equation}
which is just the usual definition of equivariance. Consequently, for the loss we have
\begin{align} \label{eq:lossinv}
    \ell_{\bg\theta}(x,y) = \ell(\N_{\bg\theta}(x),y)= l(g_Y\N_{\theta}(g_X^{-1}x), y)
    =\ell(\N_{\theta}(g_X^{-1}x),g_Y^{-1}y) = \ell_{\theta}(g_X^{-1}x, g_Y^{-1}y)\,.
\end{align}

\begin{remark} \label{rem:param_vs_functions}
   All practically used equivariant neural networks are constructed out of equivariant layers, i.e.\ $\theta\in\E$ in our language. However in general, that a network is equivariant does not imply that $\theta\in\E$. As shown in \citet{shahverdi2026identifiable}, if the parameters are \emph{identifiable} (intuitively, this means that the map $\theta \to f_\theta$ is injective up to symmetries), there is an equivalence between equivariant layers and equivariant functions.
\end{remark}

We assume the group \(G\) to be finite. We will furthermore assume that the representation $\bar{\rho}$ is unitary with respect to the Euclidean inner product. Then, 
we can consider the orthogonal complement \(T\e\) of \(T\E\). Furthermore, the orthogonal projection of a vector \(v\in T\Lc\) to \(T\E\) is given by \citep[Prop. 2.8]{fulton1991}
\begin{equation}
    p:T\Lc\rightarrow T\E\,,
    \quad
    v\mapsto \frac{1}{|G|}\sum_{g\in G}\bg v\,.
    \label{eq:e-proj}
\end{equation}

\textbf{Learning Algorithm.}
We train the network with respect to the loss function with stochastic gradient descent (SGD).
With learning rate \(\eta\) and batch size \(B\), the training step could be formulated as
\begin{align}
    \theta_{t+1}=\theta_t-\eta \frac{1}{B}\sum_{i=1}^B\nabla_\theta{\ell}_{\theta_t} (X_i, Y_i) \label{eq:gd},
\end{align}
where \((X_i, Y_i)\) are independent random variables identically distributed with \((X, Y)\). The sample gradient $\nabla_\theta l_{\theta_t}(X,Y)$ has mean and covariance
\begin{align}
    \mu({\theta_t})&=\int_\Omega \nabla_\theta\ell_{\theta_t}(x,y)\dd\mu_{X,Y}=\nabla_\theta L(\theta_t)\\
    {\Sigma}({\theta_t})&= \int_\Omega \bigl(\nabla_\theta \ell_{\theta_t}(x,y)-\mu(\theta_t)\bigr)\bigl(\nabla_\theta\ell_{\theta_t}(x,y)-\mu(\theta_t)\bigr)^\top\dd \mu_{X,Y}\,.\label{eq:cov-mat}
\end{align}

\textbf{Stochastic weight averaging.}
 The stochasticity of batch sampling provides an opportunity for the training process to explore wide basins in the loss landscape, but it also results in noise that prevents the trajectory from reaching the optimal solution. Stochastic weight averaging can be viewed as an estimator to approximate this solution, by taking the mean of the previous weights
 \begin{equation}
     \label{eq:weight_averaging}
     \bar\theta_T = \frac{1}{T}\sum_{i=1}^{T}\theta_{t_i}\,.
 \end{equation}
 
 \textbf{Performance boost.} To quantify the performance boost from stochastic weight averaging, we compare the expected loss before and after averaging. For a given starting iterate $\theta_0$
 of the averaging window, define
\begin{equation}
R(T) = \frac{\Eb[L(\theta_T)]}{\Eb[L(\bar\theta_T)]}\,,\label{eq:r-loss}
\end{equation}
where the expectation is over the batch sampling (equivalently, over the SGD trajectory). In particular, $R(T)>1$ indicates that averaging improves the loss, with larger values indicating a greater boost.

\section{Theoretical Results}
\label{sec:theory}
 As we outlined above, we will analyse stochastic gradient descent \eqref{eq:gd}. For simplicity, we will ignore effects due to finite step size and instead consider the stochastic differential equation
 \begin{equation}
    \dd \theta_t = - \nabla L(\theta_t) \dd t + \sqrt{\tfrac{\eta}{B}}P(\theta_t) \dd W_t\,, \label{eq:stochapp}
\end{equation}
where \( \dd W(t) \) is Brownian motion with covariance \(I\), and \(P(\theta)P(\theta)^T=\Sigma(\theta)\). For small values of $\eta$, \eqref{eq:stochapp} is a good approximation of the SGD dynamics. 
 A heuristic argument for this is as follows: The sample gradient $\nabla_\theta{\ell}_{\theta_t} (X, Y) $ has mean $\nabla_\theta L({\theta_t})$ and covariance ${\Sigma}({\theta_t})$. By the central limit theorem, $  \frac{1}{B}\sum_{i=1}^B\nabla_\theta{\ell}_{\theta_t} (X_i, Y_i)$ approximates a Gaussian distribution with mean $\nabla_\theta L({\theta_t})$ and covariance $\tfrac{1}{B}{\Sigma}({\theta_t})$. Adding such increments to the variable $\theta_t$ (as we do in SGD) hence corresponds essentially to an Euler-Maruyama scheme to solve \eqref{eq:stochapp}. The limit can also be justified more rigorously -- see \cite[Corollary 10]{li2019stochastic}. This motivates exclusively analyzing \eqref{eq:stochapp} in the following.
\begin{remark}
    To simplify the notation, we will normalize $\sqrt{\tfrac{\eta}{B}}=1$ from here on.
\end{remark}

\subsection{Equivariance properties of the stochastic gradients}
\label{sec:local-mininmum}
When training on augmented data, the Hessian $\nabla^2 L$ and covariance matrix $\Sigma$ satisfy the following equivariance relations. They will be crucial to our entire analysis.

\begin{lemma}\label{lemma:cov}
    For any \(\theta\in\Lc\), the following transformation rules hold,
    \begin{equation}
        \bg \Sigma(\theta)\bg^\top = \Sigma(\bg\theta), \quad \bg \nabla^2L(\theta)\bg^\top = \nabla^2L(\bg \theta). 
    \end{equation}
\end{lemma}
\begin{proof}
Applying the chain rule to \eqref{eq:lossinv} implies $\bg^{T}\nabla \ell_{\theta}(x,y)|_{\theta=\bar{g}\theta}= \nabla \ell_\theta(g_X^{-1}x,g_Y^{-1}y)$, or equivalently
\begin{align} \label{eq:chain1}
    \nabla \ell_{\theta}(x,y)\vert_{\theta = \bg \theta} = \bg \nabla \ell_\theta(g_X^{-1}x,g_Y^{-1}y).
\end{align}
Applying the chain rule one more time yields  
\begin{align}\label{eq:chain2}
   \nabla^2 \ell_{\theta}(x,y)|_{\theta=\bar{g}\theta}  = \bg\nabla^2 \ell_\theta(g_X^{-1}x,g_Y^{-1}y)\bg^\top.
\end{align}
Because the data is augmented, $(x,y) \sim  (g_X^{-1}x,g_Y^{-1}y)$, and hence, using~\eqref{eq:chain1} and~\eqref{eq:chain2}, we obtain
\begin{align}
    \Sigma(\theta) |_{\theta=\bar{g}\theta}&= \Eb [\nabla\ell_{\theta}(x,y) \nabla \ell_{\theta}(x,y)^\top|_{\theta=\bar{g}\theta}] = \Eb [\bg\nabla\ell_{\theta}(g_X^{-1}x,g_Y^{-1}y) \nabla \ell_{\theta}(g_X^{-1}x,g_Y^{-1}y)^\top\bg^\top] \\
    &= \bg \Eb [\nabla\ell_{\theta}(x,y) \nabla \ell_{\bg\theta}(x,y)^\top]\bg^\top =\bg \Sigma( \theta) \bg^\top, \\
    \nabla^2L( \theta)|_{\theta=\bar{g}\theta} &= \Eb(  \nabla^2 \ell_{\theta}(x,y)|_{\theta=\bar{g}\theta}) =\Eb[\bg\nabla^2 \ell_\theta(g_X^{-1}x,g_Y^{-1}y)\bg^\top] \\
    &=\bg\Eb[\nabla^2\ell(x,y)]\bg^\top 
    = \bg \nabla^2L(\theta) \bg^\top.\qedhere
\end{align}
\end{proof}
Consider the flow $\Phi_t$ associated to \eqref{eq:stochapp}, that maps the initial probability measure $\mu_0$ of $\theta$ to the measure $\mu_t$ at time $t$, i.e.\ $\mu_t = \Phi_t^*\mu_0$. A first consequence of Lemma~\ref{lemma:cov} is the following equivariance of $\Phi$,
\begin{align} 
    \Phi_t^*\bg^* &= \bg^*\Phi_t^*\,. \label{eq:eqflow}
\end{align}
Here, $f^*$ denotes the pushforward operation of a map $f$. The discrete version of this equivariance guarantee was in \cite{nordenfors2024ensembles} shown to imply a weak equivariance guarantee for an \emph{ensemble} of neural networks trained with SGD. Concretely, if we initialize an ensemble using an invariant distribution, the same will be true for all times $t$, which in turn will imply that the ensemble mean will be equivariant at all times. 

\begin{proposition} \label{prop:ensemble}
Let $\mu_0$ be an invariant parameter distribution: $\bg^*\mu_0 = \mu_0$ for all $g \in G$. Let $\mu_t = \Phi_t^*\mu_0$ be the distribution obtained by independently training an ensemble of networks using \eqref{eq:stochapp} for a time $t$. Then, the ensemble-averaged network
\begin{equation}
\overline{\N}_t(x) = \mathbb{E}_{\theta\sim \mu_t}[\N_{\theta}(x)]\,.\label{eq:net-mean}
\end{equation}
is $G$-equivariant for every $t>0$.
\end{proposition}
\begin{proof}
    See Appendix~\ref{app:proof_ensemble}.
\end{proof}

Note that this result is not very informative for a single model -- its initial distribution $\mu_0$ is a Dirac delta, and hence only invariant if it is concentrated in a point $\theta_0\in \E$ (applying the theorem in this context in fact recovers the results from \cite{nordenfors2025optimization} about single networks). Still, it is interesting that \eqref{eq:eqflow} persists in our stochastic setting, and we therefore provide a proof -- based on analysing the Fokker--Planck equation -- in Appendix~\ref{app:proof_ensemble}.

\subsection{Dynamics close to equivariant minima and the Ornstein-Uhlenbeck process}
We will now move on to study SWA~\eqref{eq:weight_averaging} more specifically. Since stochastic averages are taken along single training runs, we need to analyse the solution paths of \eqref{eq:stochapp}. We will assume that solutions converge towards equivariant minima since they are trained on augmented data. Put differently, we will study the dynamics of \eqref{eq:stochapp} near a local minimum $\theta_*$ of the cumulative loss $L$, which we furthermore assume is equivariant in the strict sense $\theta_*\in \E$ (cf. Remark \ref{rem:param_vs_functions}). To simplify the notation, we will shift the coordinates, so that $\theta_*=0$.

Near $\theta_*=0$, we can approximate $L$ with the quadratic function $L(\theta) = \tfrac{1}{2} \theta^\top H\theta$, with $H = \nabla^2 L(\theta_*)$. We can furthermore make the approximation $P(\theta)\approx P(\theta_*)=:P$. The dynamics \eqref{eq:stochapp} then become
\begin{equation}
    \dd \theta_t = -H\theta_t + P\dd W_t\,.
\end{equation}
This is the well-known \emph{Ornstein--Uhlenbeck} process. Its solution has a closed form: If the process is initialized at $\theta_0$, we have 
\begin{equation} \label{eq:OU}
    \theta_t = e^{-Ht}\theta_0 + \int_0^t e^{-H(t-s)}P \dd W_s\,.
\end{equation}

Since we are studying the dynamics at an equivariant point $\theta_*$, $H$ and $P$ decompose into block matrices with respect to the decomposition $\E \oplus \E^\perp$, as specified by the following 
\begin{lemma}
\label{lem:HP-decomp}
The matrices $H$ and $P$ commute with $\bg$. Consequently, they have a block diagonal structure: They map $\E$ to $\E$ and $\E^\perp$ to $\E^\perp$.
\end{lemma}
\begin{proof}Since $\theta_*\in \E$, we have $\bg\theta_*=\theta_*$. Together with Lemma \ref{lemma:cov}, this shows that $\bg^T H\bg=H$, $\bg^T \Sigma(\theta_*)\bg= \Sigma(\theta_*)$, which is the first claim.
   
    Let $v\in \E$ and $w\in\E^\perp$. To prove the second claim, we note that by definition of $\E$, we have $\bg v = v$. Furthermore, the projection~\eqref{eq:e-proj} yields
    \begin{align}
        \frac{1}{|G|}\sum_{g\in G} \bg w  = 0\,.
    \end{align}
    We therefore obtain
    \begin{align}
        v^\top Hw =    \frac{1}{|G|}\sum_{g\in G} v^\top \bg^\top H \bg w  = \frac{1}{|G|}\sum_{g\in G} (\bg v)^\top  H \bg w = v^\top H\frac{1}{|G|}\sum_{g\in G}\bg w = 0\,,
    \end{align}
  The argument for $v^\top \Sigma w=0$ proceeds in the same manner. Since $\mathcal{E}$ and $\mathcal{E}^\perp$ are orthogonal complements, we may choose orthonormal bases $\{e_i\}$ for $\mathcal{E}$ and $\{e_i^\perp\}$ for $\mathcal{E}^\perp$, so that $\{e_i\} \cup \{e_i^\perp\}$ forms an orthonormal basis for the ambient space. In this basis, $\Sigma$ takes the block-diagonal form $\Sigma = \mathrm{diag}\{\Sigma_{\mathcal{E}}, \Sigma_{\mathcal{E}^\perp}\}$, where $\Sigma_{\mathcal{E}}$ and $\Sigma_{\mathcal{E}^\perp}$ denote the restrictions of $\Sigma$ to $\mathcal{E}$ and $\mathcal{E}^\perp$, respectively. We choose $P$ to be block-diagonal in the same basis, $P = \mathrm{diag}\{P_{\mathcal{E}}, P_{\mathcal{E}^\perp}\}$, with $P_{\mathcal{E}} P_{\mathcal{E}}^\top = \Sigma_{\mathcal{E}}$ and $P_{\mathcal{E}^\perp} P_{\mathcal{E}^\perp}^\top = \Sigma_{\mathcal{E}^\perp}$. Then $PP^\top = \Sigma$, and for any $v \in \mathcal{E}$, $w \in \mathcal{E}^\perp$, we have $v^\top P w = 0$.
\end{proof}
The previous lemma, and basic properties of the matrix exponential, immediately show that the solution \eqref{eq:OU} also decomposes: We get that $\theta = (\theta^\E, \theta^{\E^\perp})$, with 
\begin{align}
    \theta^\E_t &= e^{- H_\E t}\theta^\E_0 + \int_0^t e^{-H_\E(t-s)}P_\E \dd W_s^\E \nonumber\\
    \theta^{\E^\perp}_t &= e^{- H_{\E^\perp} t}\theta^{\E^\perp}_0 + \int_0^t e^{-H_{\E^\perp}(t-s)}P_{\E^\perp} \dd W_s^{\E^\perp}\,, \label{eq:Eperpsol}
\end{align}
for independent Brownian motions $W_s^\E, W_s^{\E^\perp}$ on $\E$ and $\E^\perp$, respectively. 

With the block structure of the Hessian, we introduce a quantity to measure the equivariance of the network. Remember that $\N_\theta$ is equivariant if $\theta\in \E$. Hence, a reasonable measure of the ``non-equivariance'' of a parameter is the norm of the projection of $\theta$ onto $\E^\perp$. By Lemma~\ref{lem:HP-decomp} (and the fact that $H\succ 0$ at the strict local minimum $\theta_*$), this is proportional to the quantity
\begin{equation} \label{eq:nonequivariance}
    L_{\E^\perp}(\theta_t) = \left(\theta_t^{\E^\perp}\right)^\top H_{\E^\perp}\theta_t^{\E^\perp}\,.
\end{equation}
Here, we measure the norm of the non-equivariant component of $\theta$ in the metric defined by $H_{\E^\perp}$ which captures the length scale of the curvature of the loss landscape. I.e.\ if $L_{\E^\perp}$ is small, then the non-equivariant component of the parameters is small in comparison to the length scale of the loss and the network is close to equivariance. 

Analogous to~\eqref{eq:r-loss}, we introduce a quantity to measure the gain in equivariance from performing SWA. Concretely, we define it as the ratio of $L_{\E^\perp}$ evaluated on averaged weights and on unaveraged weights at the end of training,
\begin{equation}
    R_{\E^\perp}(T) = \frac{\Eb[L_{\E^\perp}(\theta_T)]}{\Eb[L_{\E^\perp}(\bar{\theta}_T)]}\,.
    \label{eq:r-perp}
\end{equation}
We will in the following argue that $R_{\E^\perp}(T)$ will be large (corresponding to a large gain in equivariance) in two settings: First, we will look at the case of weights sampled far apart from each other in time. Next, we will consider a more realistic scenario of dense samples, in a specialized setting of invariant classification problems. In the latter case, we will give detailed analysis of the behaviour of $R_{\E^\perp}(T)$ in the NTK limit.

\subsection{Well-separated samples}
Since we are at a strict local minimum, the operators $H_\E$ and $H_{\E^\perp}$ are positive definite. Therefore, $e^{-H_\E t}$ and $e^{-H_{\E^\perp}t}$ will for large values of $t$ be vanishingly small. This together with the decomposition~\eqref{eq:Eperpsol} shows that $\theta_t^\E$ and $\theta_t^{\E^\perp}$ will be approximately independent of each other, and, more importantly, approximately independent of the initialization. Hence, if the sampling times $t_i=(i\Delta_t)_{i=1}^T$ in~\eqref{eq:weight_averaging} are well separated in time ($\Delta_t$ sufficiently large), the sequence of samples $\theta^{\E^\perp}_{t_i}$ will be approximately i.i.d. samples of the limit distribution
\begin{align}
    \theta_{\Delta_t}^{\E^\perp} = \int_0^{\Delta_t} e^{-H_{\E^\perp}(t-s)}P_{\E^\perp} \dd W_s^{\E^\perp}\,, 
\end{align}
which is a Gaussian with mean zero and covariance matrix 
\begin{align}
    \sigma_{\Delta_t} = \int_0^{\Delta_t} e^{-H_{\E^\perp}(t-s)}P_{\E^\perp}P_{\E^\perp}^T e^{-H_{\E^\perp}(t-s)}\dd s\,.
\end{align}
The same argument applies to $\theta^\E_{t_i}$. As an i.i.d. mean of Gaussians, the stochastic weight average $\theta^{\mathrm{avg.}}_t$ also becomes Gaussian distributed for large $T$, with covariance $\tfrac{1}{T}\sigma_{\Delta_t}$. Since the covariance of $\theta_t$ does not scale with $T$, we therefore conclude
\begin{equation}
    R(T\Delta_t) \sim T\,. \qquad \text{(separated samples)}
\end{equation}
Note that this result, while it conveys the main intuition behind why SWA improves equivariance, is not readily applicable to practical training runs. For one, it is hard to estimate a priori for how long we need to wait between taking samples in order for our approximations above to be valid. Longer wait times would furthermore increase training time. In the next section, we will therefore give a more detailed analysis that does not rely on the assumption of well-separated samples.

\subsection{Dense samples in a classification setting} 
Let us now consider the complete opposite limit of taking samples well separated in time: taking them infinitesimally often. That is, make the approximation
\begin{align}
      \overline{\theta}_T = \frac{1}{T}\sum_{i=1}^T \theta_{t_i} \approx \frac{1}{T}\int_0^T \theta_t \dd t.\label{eq:cont-avg}
\end{align}
Using standard results of stochastic calculus, one sees that $\overline{\theta}_T$ also here becomes Gaussian distributed, with mean zero, but with a more complicated expression for the covariance:
\begin{align}
    \mathrm{Cov}(\overline{\theta}_T) = \frac{1}{T^2}\int_0^T\int_0^T \int_0^{\min(s,t)} e^{-H(t-r)}PP^Te^{-H(s-r)}\,\dd r\, \dd s\, \dd t.
\end{align}
To give a direct comparison between $L(\theta_t)$ and $L(\overline{\theta}_t)$ is therefore harder than in the last section. It will depend highly on the interplay of the eigendecompositions of $H$ and $P$, and we will therefore not do it in full generality.
We instead again specialize, this time to \emph{group-invariant classification problems}. That is, the group acts invariantly on the class label $y \in \{1,\dots,n_Y\}$. As usual in such a setting, we use our network outputs to define a posterior probability distribution 
\begin{equation}
    p_\theta(y \, \vert \, x)= \mathrm{softmax}(\N_\theta(x))
\end{equation} 
over the label space $\{1, \dots, n_Y\}$. This distribution could be interpreted as a vector in the Euclidean space $\Rb^{n_Y}$. We train the network using a cross-entropy loss:
\begin{align}
    \ell_\theta(x,y) = -\log (p_\theta(y|x)).
\end{align}

We further restrict the analysis to the loss landscape near a well-fitted minimum. We claim that in this regime, the matrix $\Sigma=PP^T$ is approximately equal to the Hessian $H$. By well-fitted, we mean that the residual error $\varepsilon(x) = p_\theta(x) - e_{y(x)} \in \Rb^{n_Y}$ between the predicted distribution $p_\theta(x)= \sum_{y=1}^{n_Y} p_\theta(y \, \vert x)e_y$ and the data distribution $p_{\mathrm{data}}(y \, \vert x)=e_{y(x)}$ is small. Here, $e_k\in\mathbb{R}^{n_Y}$ is the unit vector with components $(e_k)_i=\delta_{ik}$, and $y(x)$ is the true label class for the sample $x$.

To formalize this approximation, let $J(x) = \nabla_\theta \N_\theta(x)$ be the Jacobian of the network. Using the chain rule, the gradient of the loss is
\begin{align}
    \nabla_\theta L(\theta ) 
    &= \int_{V_X} J(x)^\top(p_\theta(x)-e_{y(x)})\dd \mu_X = \int_{V_X} J(x)^\top\varepsilon(x)\dd \mu_X\,.
\end{align}
Consequently, the Hessian takes the form\footnote{Here, we use the notation $\nabla^2_\theta \N_\theta(x)$ for the tensor $\nabla^2_\theta\N(\theta)_{ijk} = \partial_{\theta_i}\partial_{\theta_j}\N_\theta(x)_k$, and denote by $\nabla_\theta^2 \N_\theta(x)\varepsilon(x)$ the contraction with $\varepsilon$ over the $k$-dimensional space.}
\begin{align} 
   H = \nabla^2 L(\theta_*) =& \underbrace{ \int_{V_X} J(x)^\top \Lambda(x) J(x) \dd \mu_X \;}_{\mathcal{G}}+\; \underbrace{\int_{V_X}\nabla_\theta^2 \N_\theta(x)\varepsilon(x)\dd \mu_X}_{\mathcal{R}}\,, \label{eq:Hessianform}
\end{align}
where we defined $\Lambda(x) = \mathrm{diag}(p_\theta(x)) - p_\theta(x)p_\theta(x)^\top$. Note that the residual term $\mathcal{R}$ is linear in the small error $\varepsilon$. Now consider the covariance matrix $\Sigma$. Applying \eqref{eq:cov-mat} and the chain rule, we obtain
\begin{align}
      \Sigma &= \int_{V_X}  \,\nabla_\theta \log p_\theta(y|x)\,\nabla_\theta \log p_\theta(y|x)^\top\,  \dd \mu_X = \int_{V_X} J(x)^\top \left[\varepsilon(x)\varepsilon(x)^\top\right] J(x)\, \dd \mu_X\,. \label{eq:sigmaform}
\end{align}
By expanding the bracket $\varepsilon(x)\varepsilon(x)^\top = \mathrm{diag}(e_{y(x)})- 2e_{y(x)}p_\theta(x)^\top + p_\theta(x)p_\theta(x)^\top$, we see that it equals $\Lambda (x) - \Delta(x)$, where $\Delta(x) = \mathrm{diag}( \varepsilon(x))-2\varepsilon(x) p_\theta(x)^\top$. Combining \eqref{eq:Hessianform} and \eqref{eq:sigmaform}, the difference between the two matrices is
\begin{align}
    H- \Sigma 
    = \mathcal{R} + \int_{V_X}\sum_y J(x)^\top\Delta(x) J(x) \dd \mu_X .
\end{align}
Since both terms on the right-hand side are linear in $\varepsilon$, we have $H \approx \Sigma$ when $\varepsilon \ll 1$. Furthermore, this structure allows us to formalize the relation between $L_{\mathrm{eq}}$ and $L_{\E^\perp}$.

Another direct outcome is that the two have a common eigenbasis, in which the dynamics decompose further into one-dimensional solutions
\begin{align}
    \theta_t^i = e^{-\lambda_it}\theta_0^i + \int_0^t e^{-\lambda_i(t-s)}\sqrt{\lambda_i}\,\dd B_s^i\,.
    \label{eq:1d-sol}
\end{align}
Here, $B_s^i$ are independent, one-dimensional Brownian motions and $\lambda_i$ are the eigenvalues of $H$ (the eigenvalues of $P$ are then $\sqrt{\lambda_i}$). This gives the following simplified form of the averaged process~\eqref{eq:cont-avg}:
\begin{align}
   \bar\theta^i_T = &\frac{1}{T}\int_0^T\theta^i_t\dd t = \frac{1}{T}\int_0^T e^{-\lambda_it}\theta_0^i \dd t+ \frac{1}{T}\int_0^T\int_0^t e^{-\lambda_i (t-s)} \sqrt{\lambda_i }\,\dd B_s^i\,\dd t\\
     =&\frac{1-e^{-\lambda_iT }}{\lambda_i T}\theta_0^i + \frac{1}{\sqrt{\lambda_i}T}\int_0^T (1 - e^{-\lambda_i(T-s)})\,\dd B_s^i\,.
     \label{eq:1d-avg-sol}
\end{align}
We will now use~\eqref{eq:1d-sol} and~\eqref{eq:1d-avg-sol} to derive an expression for $R_{\E^\perp}(T)$ from~\eqref{eq:r-perp}. Using the Itô isometry, we obtain
\begin{align}
    \Eb[\tfrac12\theta_T^i (H\theta_T)^i] &= \frac{1}{2}\lambda_i \Eb[(\theta_T^i)^2] = \frac{1}{2}\lambda_i e^{-2\lambda_iT}(\theta_0^i)^2 +  \frac{1}{2}\lambda_i^2 \int_0^T e^{-2\lambda_i(T-s)} \dd s\\
    &= \frac{\lambda_i}{2} \left( e^{-2\lambda_iT}(\theta_0^i)^2 +  \frac{1}{2} (1-e^{-2\lambda_iT})\right)\\
    \Eb[\tfrac12\bar\theta_T^i (H\bar\theta_T)^i] &= \frac{1}{2}\lambda_i \Eb[(\bar\theta_T^i)^2] =  \frac{(1-e^{-\lambda_i T})^2}{2\lambda_iT^2}(\theta_0^i)^2 +  \frac {1}{2T^2}\int_0^T (1-e^{-2\lambda_i(T-s)})^2 \dd s\\
    &=  \frac{(1-e^{-\lambda_i T})^2}{2\lambda_iT^2}(\theta_0^i)^2 + \frac{1}{2}\left(\frac{1}{T} - \frac{1-e^{-2\lambda_i T}}{\lambda_iT^2 }+ \frac{1-e^{-4\lambda_i T}}{4\lambda_iT^2}\right)\,.
    \label{eq:averaged_comp}
\end{align}
The behavior of the above expressions for large $T$ depends on whether $\lambda_i T$ is large or small. For $\lambda_i T\gg 1$, we can approximate
$e^{-c\lambda_iT}\approx 0$, so that we get
\begin{align}
 \Eb[\tfrac12\theta_T^i (H\theta_T)^i] &\approx \frac{\lambda_i}{4} \sim \lambda_i \\
\Eb[\tfrac12\bar\theta_T^i (H\bar\theta_T)^i]& \approx \frac{1}{2\lambda_iT^2}(\theta_0^i)^2 + \frac{1}{2}\left(\frac{1}{T} - \frac{1}{\lambda_iT^2 }+ \frac{1}{4\lambda_iT^2}\right)\sim \frac{1}{2T}\qquad (\lambda_iT \gg 1)\,.
\end{align}
For $\lambda_i T\ll 1$, we instead have $1-e^{-c\lambda_iT} \approx c\lambda_i T $, so that
\begin{align}
 \Eb[\tfrac12\theta_T^i (H\theta_T)^i] &\approx  \frac{\lambda_i}{2}((1-2\lambda_iT) (\theta_0^i)^2 + \lambda_i T) \sim \frac{\lambda_i}{2} (\theta_0^i)^2\\
  \Eb[\tfrac12\bar\theta_T^i (H\bar\theta_T)^i] &\approx   \frac{\lambda_i}{2}(\theta_0^i)^2 + \frac{1}{2}\left(\frac{1}{T} - \frac{2}{T}+ \frac{1}{T}\right) \sim \frac{\lambda_i}{2} (\theta_0^i)^2 \qquad (\lambda_iT \ll 1)\,.
\end{align}
Summing the above expressions over all dimensions recovers the expected losses $L(\theta_T)$ and $L(\bar{\theta}_T)$. For large $T$, we obtain
\begin{align}
    \Eb[L(\theta_T)] &= \sum_{i} \Eb[\tfrac12\theta_T^i (H\theta_T)^i]\sim  \frac{1}{2}\sum_{\substack{\lambda_iT \ll 1}} \lambda_i (\theta^i_0)^2 + \frac{1}{4}\sum_{\substack{\lambda_iT \gg 1}} \lambda_i  = \frac{1}{4}\mathrm{Tr}(H) +\epsilon_T\\
    \Eb[L(\bar\theta_T)] & = \sum_{i}  \Eb[\tfrac12\bar\theta_T^i (H\bar\theta_T)^i] \sim  \frac{1}{2}\sum_{\substack{\lambda_iT \ll 1}} \lambda_i (\theta_0^i)^2 + \frac{1}{2}\sum_{\substack{\lambda_iT \gg 1}} \frac{1}{T} = \frac{1}{2T} \mathrm{N}(1/T) + \bar \epsilon_T\,,
\end{align}
where $\mathrm{N}(1/T)$ denotes the number of large eigenvalues $\lambda_i$ with $\lambda_i\gg 1/T$ and $|\epsilon_T| , |\bar\epsilon_T|\leq \sum_{\lambda_i T\ll 1} \lambda_i \to 0$, $T\to \infty$. Hence, for large $T$,
\begin{align}
    R(T) = \frac{T}{2}\cdot \frac{\mathrm{Tr}(H) + \epsilon_T}{\mathrm{N}(1/T) + \bar\epsilon_T}. \label{eq:r-as-trH}
\end{align}
Analogously, defining $\mathrm{N}_{\E^\perp}(1/T)$ as the count of large eigenvalues in $H_{\E^\perp}$, we have 
\begin{align}
    R_{\E^\perp}(T) = \frac{T}{2}\cdot \frac{\mathrm{Tr}(H_{\E^\perp}) + \epsilon_T}{\mathrm{N}_{\E^\perp}(1/T) + \bar\epsilon_T}\,. \label{eq:rperp-as-trH}
\end{align}
These results already show that, even with dense samples, both $R(T)$ and $R_{\E^\perp}(T)$ converge to infinity as $T\to \infty$. Hence, stochastic weight averaging encourages equivariance and boosts performance at the same time.

This begs the question: How much of the equivariance gain from stochastic weight averaging simply comes from a generic improvement in overall performance? To measure this, we should study the ratio between $R_{\E^\perp}(T)$ and $R(T)$, i.e.
\begin{align}
   \frac{R_{\E^\perp}(T)}{R(T)} = \frac{\mathrm{N}(1/T) + \epsilon_T'}{\mathrm{N}_{\E^\perp}(1/T) + \bar\epsilon_T}\frac{\mathrm{Tr}(H_{\E^\perp}) + \bar\epsilon_T}{\mathrm{Tr}(H) + \epsilon_T'}\,.
\end{align}
Asymptotically for $T\to \infty$, $\mathrm{N}(1/T)$ and $\mathrm{N}_{\E^\perp}(1/T)$ will count all dimensions in $\Lc$ and $\E^\perp$, respectively, and the $\epsilon$-terms converge to zero. Consequently,
\begin{align}
    \lim_{T\to \infty }  \frac{R_{\E^\perp}(T)}{R(T)} = \frac{\dim \Lc}{\dim \E^\perp }\frac{\mathrm{Tr}(H_{\E^\perp})}{\mathrm{Tr}(H)} \,.\label{eq:inf-equiv-improv-ratio}
\end{align}

In the next section, we will derive a bound on this ratio in the NTK-, or \emph{infinitely wide}, limit \citep{jacot2018}.
\subsection{A refined analysis in the NTK limit}
We begin by specifying the theoretical setting. We consider a \emph{base architecture} as above, i.e.\ with input space $X$, intermediate spaces $V_l$ and output space $Y$ on which representations $\rho_X$, $\rho_l$ and $\rho_Y$ act. We increase the network's width by a factor $N$ by taking $N$ direct sums of the intermediate spaces, i.e.\  $V_l \to V_l^{\oplus N}$. The group then acts by the direct sum representation on the enlarged spaces, $\rho_l \to \rho_l^{\oplus N} $. We will consider the limit $N\to \infty$. 

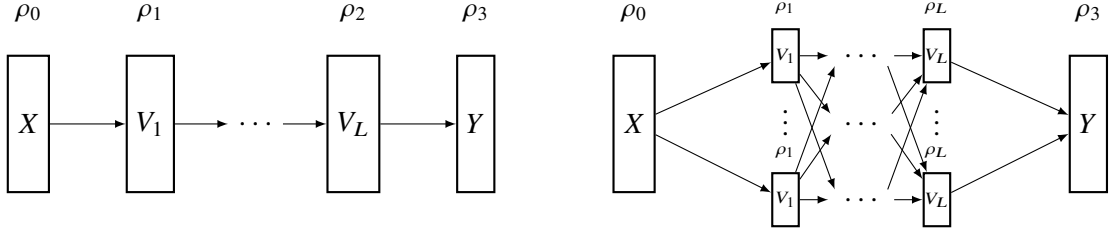
\begin{figure}
    \centering
    \begin{tikzpicture}[
    feature/.style={
        draw,
        thick,
        minimum width=0.45cm,
        minimum height=1.8cm
    },
    smallfeat/.style={
        draw,
        thick,
        minimum width=0.35cm,
        minimum height=0.7cm
    },
    rho/.style={
        above=3mm,
        font=\small
    },
     smallrho/.style={
        above=.6mm,
        font=\tiny
    },
    hdots/.style={
        below=.1mm,
    },
    >=latex,
]


\node[feature] (f0) at (0,0) {$X$};
\node[rho] at (f0.north) {$\rho_0$};

\node[feature,right=1cm of f0] (f1) {$V_1$};
\node[rho] at (f1.north) {$\rho_1$};

\node (inbet) at (3,0) {$\hdots$};

\node[feature,right=2cm of f1] (f2) {$V_L$};
\node[rho] at (f2.north) {$\rho_2$};

\node[feature,right=1cm of f2] (f3) {$Y$};
\node[rho] at (f3.north) {$\rho_3$};

\draw[->] (f0) -- (f1);
\draw[->] (f1) -- (inbet);
\draw[->] (inbet) -- (f2);
\draw[->] (f2) -- (f3);


\begin{scope}[xshift=8cm]

\node[feature] (g0) at (0,0) {$X$};
\node[rho] at (g0.north) {$\rho_0$};

\node[smallfeat] (g1a) at (2, .9) {};
\node at (2,.9) {{\tiny $V_1$}};
\node[smallfeat] (g1b) at (2,-1) {};
\node at (2,-1) {{\tiny $V_1$}};
\node[smallrho]  at (g1a.north) {$\rho_1$};
\node[hdots] at (g1a.south) {$\vdots$};
\node[smallrho] at (g1b.north) {$\rho_1$};

\node (inbeta) at (3,.9) {$\hdots$};
\node (inbeta1) at (2.7,.9) {};
\node (inbeta2) at (3.3,.9) {};

\node (inbetc) at (3,0) {$\hdots$};
\node (inbetc1) at (2.7,0) {};
\node (inbetc2) at (3.3,0) {};

\node (inbetb) at (3,-1) {$\hdots$};
\node (inbetb1) at (2.7,-1) {};
\node (inbetb2) at (3.3,-1) {};

\node[smallfeat] (g2a) at (4, .9) {};
\node at (4,.9) {{\tiny $V_L$}};
\node[smallfeat] (g2b) at (4,-1) {};
\node at (4,-1) {{\tiny $V_L$}};
\node[smallrho]  at (g2a.north) {$\rho_L$};
\node[hdots] at (g2a.south) {$\vdots$};
\node[smallrho] at (g2b.north) {$\rho_L$};

\node[feature] (g3) at (6,0) {$Y$};
\node[rho] at (g3.north) {$\rho_3$};

\draw[->] (g0) -- (g1a);
\draw[->] (g0) -- (g1b);

\draw[->] (g1a) -- (inbeta1);
\draw[->] (g1b) -- (inbetb1);
\draw[->] (g1a) -- (inbetb1);
\draw[->] (g1b) -- (inbeta1);
\draw[->] (g1a) -- (inbetc1);
\draw[->] (g1b) -- (inbetc1);

\draw[->] (inbeta2) -- (g2a);
\draw[->] (inbetb2) -- (g2b);
\draw[->] (inbeta2) -- (g2b);
\draw[->] (inbetb2) -- (g2a);
\draw[->] (inbetc2) -- (g2b);
\draw[->] (inbetc2) -- (g2a);

\draw[->] (g2a) -- (g3);
\draw[->] (g2b) -- (g3);

\end{scope}

\end{tikzpicture}
    \caption{Definition of the NTK limit we consider here. On the left is the 'base architecture', on the right is the extended architecture by a factor $N$.}
    \label{fig:NTK-limit}
    \vspace{-2em}
\end{figure}

 As finite-dimensional representations of a finite group, the actions $\rho_l$ on the base intermediate spaces decompose into direct sums of irreducible representations (irreps) $\rho_\lambda$ as~\citep{fulton1991}
\begin{align}
\rho_l = \bigoplus_{\lambda} \rho_\lambda^{\oplus m_{\lambda}^l}\,,\label{eq:irrep-decomp}
\end{align}
where $m^l_\lambda\in\mathbb{N}_0$ denotes the multiplicity of the irrep $\lambda$ and the sum runs over all irreps of $G$. In particular,  \(n_l = \dim V_l = \dim\rho_l = \sum_\lambda m_\lambda^l \dim \rho_\lambda\). The extended representations decompose in the same way, the only thing that changes are the multiplicities: $m_\lambda^l \to Nm_\lambda^l$.

Now, the base equivariant parameters $\theta\in \E$ satisfy $g_{l+1}W^l = W^l g_l$. That is, they belong to the space $\mathrm{Hom}(\rho_l, \rho_{l+1})$. By (the real version of) Schur's lemma, that space is
\begin{equation}
    \mathrm{Hom}(\rho_l, \rho_{l+1}) \cong \{\oplus_\lambda A^l_\lambda \otimes E_{\lambda}:A_\lambda^l\in \Rb^{m_\lambda^{l+1}\times m_{\lambda}^l}\}, \label{eq:schur-decomp} 
\end{equation}
where $E_\lambda$ is the space of equivariant endomorphisms on $\rho_\lambda$. By a well-known theorem of Frobenius, $E_\lambda$ has real dimension $1$, $2$ or $4$, depending on whether the irrep $\rho_\lambda$ is of \emph{real}, \emph{complex} or \emph{quaternionic} type \citep{fulton1991}. To keep the exposition readable, we will in the following concentrate on the case of all $\rho_\lambda$ being real, so that $E_\lambda=\mathrm{span}\{\id\}$. We will comment on the minor adjustment needed in the general case at the very end.

In the case of $E_\lambda = \mathrm{span}\{\id\}$, the dimension of the space of equivariant parameters at layer $l$ is
\begin{align}
    \dim\E^l_{\mathrm{base}}=\sum_\lambda m_\lambda^{l+1}m_{\lambda}^l\,.
\end{align}
The dimension of the space of all parameters at base layer $l$ is on the other hand given by
\begin{align}
    \dim\Lc^l_{\mathrm{base}}=n_{l+1}n_l = \left(\sum_\lambda m_\lambda^{l+1}\dim \rho_\lambda \right)\left(\sum_{\lambda'} m_{\lambda'}^l \dim\rho_{\lambda'}\right)\,.
\end{align}
Since $G$ is a finite group, $\sum_\lambda (\dim \rho_\lambda)^2=|G|$ (see e.g.~\cite{fulton1991}) and we obtain the following bound by using Cauchy--Schwarz:
\begin{align}
    \dim \Lc^l_{\mathrm{base}} \leq &|G|\left(\sum_\lambda (m_\lambda^{l+1})^2\right)^\frac{1}{2}\left(\sum_{\lambda} (m_{\lambda}^{l})^2\right)^\frac{1}{2}
\end{align}
We now assume that the intermediate spaces all scale with the same rate in $N$ so that the ratios of their widths stay fixed. Formally, $r^{-1}\leq m_\lambda^{l+1} / m_\lambda^l\leq r$ for all $\lambda$ and $l>1$ for some $r\geq 1$. This gives
\begin{align}
    \dim \Lc^l_{\mathrm{base}}\leq r|G|\dim \E_{\mathrm{base}}^l\,.
\end{align}
Since $\dim \E_{\mathrm{ext}}^l = N^2\dim \E_{\mathrm{base}}^l$ for $1\leq l\leq L-1$, and similarly for $\dim \Lc_{\mathrm{ext}}^l$, this inequality remains valid as we take the limit. However, it does not apply to the first and last layers whose parameters involve the fixed input and output representations. We circumvent this issue by a simple argument,
\begin{align}
\dim \Lc^0_{\mathrm{ext}}&= n_X \cdot (Nn_1) = \frac{n_X}{Nn_2}\cdot (Nn_1)\cdot(Nn_2) = \frac{n_X}{Nn_2}\cdot \dim \Lc^1_{\mathrm{ext}} \\ 
&\leq \frac{n_X}{Nn_2}\sum_{l=1}^{L-1}\dim \Lc^l_{\mathrm{ext}}\leq \frac{n_Xr|G|}{Nn_2}\sum_{l=1}^{L-1}\dim \E^l_{\mathrm{ext}}\leq \frac{n_Xr|G|}{N}\dim\E_{\mathrm{ext}}\,
\end{align}
and similarly
\begin{align}
    \dim \Lc^L_{\mathrm{ext}}&=  (Nn_L)n_Y = \frac{n_Y}{Nn_{L-1}}\cdot (Nn_{L-1})\cdot(Nn_L) = \frac{n_Y}{Nn_{L-1}}\cdot \dim \Lc^{L-1}_{\mathrm{ext}} \\ 
&\leq \frac{n_Y}{Nn_{L-1}}\sum_{l=1}^{L-1}\dim \Lc^l_{\mathrm{ext}}\leq \frac{n_Xr|G|}{Nn_{L-1}}\sum_{l=1}^{L-1}\dim \E^l_{\mathrm{ext}}\leq \frac{n_Yr|G|}{N}\dim\E_{\mathrm{ext}}\,
\end{align}

We conclude that 
\begin{equation}
    \frac{\dim \E_{\mathrm{ext}} }{ \dim \Lc_{\mathrm{ext}}} =\frac{\dim\E_{\mathrm{ext}}}{\sum_{l=1}^{L-1}\dim\Lc_l + \dim \Lc^0_{\mathrm{ext}} + \dim \Lc^L_{\mathrm{ext}}}\geq \frac{1}{(1+(n_X+n_Y)N^{-1})r|G|}\sim \frac{1}{r|G|}
\end{equation}
for large $N$. Therefore,
\begin{equation}
    \frac{\dim \E^\perp_{\mathrm{ext}}}{\dim\Lc_{\mathrm{ext}}}\leq 1-\frac{ 1}{r|G|}
    \qquad\Rightarrow\qquad
    \frac{\dim\Lc_{\mathrm{ext}}}{\dim \E^\perp_{\mathrm{ext}}}\geq \frac{r|G|}{r|G|-1}
    \label{eq:dimratio}
\end{equation}
for large $N$. Note that $r$ is a constant that depends on the network structure but does not scale with $N$.
\subsubsection{Ratio of the traces in the infinite-width limit}
From here we consider the network training in the infinite-width limit. Namely, denoting the network width scaling as $N$, the learning rate is $O_N(1)$, while the weights are initialized as $W_{ij}^l\sim\N(0,{c_l^2})$. Under this condition and mild assumptions on the nonlinearity, the pre-activations $h^l(x)$ converge to Gaussian processes with covariance kernel
\begin{equation}
\Sigma^{(l)}(x,x') = \fs(x,x')\mathrm{Id}_{Nn_l}.    
\end{equation}
Furthermore,
the neural tangent kernel (NTK), defined as 
\begin{equation}
    {K}(x,x') = J(x) J(x')^\top,
\end{equation}
converges in operator norm to ${K}_\infty(x,x')= \mathsf{K}(x, x')\, I_Y$ for some symmetric positive semidefinite kernel $\mathsf{K}: \mathbb{R}^{{n_X}} \times \mathbb{R}^{n_X} \to \mathbb{R}$ 
with $\Tr (\mathsf{K}) = \int_X \mathsf{K}(x, x)p_\mathrm{data}(x)\dd x < \infty$ \citep{jacot2018}. 

The limiting kernel can be obtained recursively layer-by-layer. For an MLP with nonlinearities $\sigma_l$, the recursion relations are given by
\begin{align}
\mathsf{\Sigma}^{(0)}(x,x') &= \frac{c_0^2}{n_X} x^\top x'\,,  \qquad
 \mathsf{K}^{(0)}(x,x') = 0\label{eq:ntk-recursion-init}\\
 \fs^{(l+1)}(x,x') &=  \, c^2_{l+1}\mathbb{E}_{u,u'\sim\fs^{(l)}(x,x')}\big[\sigma_{l}(u)\sigma_{l}(u')\big],\\
 \dot{{\fs}}^{(l+1)}(x,x') &=  \, c^2_{l+1}\mathbb{E}_{u,u'\sim\fs^{(l)}(x,x')}\big[\sigma_{l}'(u)\sigma_{l}'(u')\big], \\
\fk^{(l+1)}(x,x') &= \fk^{(l)}(x,x') \dot{\fs}^{(l+1)}(x,x') + \fs^{(l+1)}(x,x'),\label{eq:ntk_recursive}
\end{align}
where, by $\mathbb{E}_{u,u'\sim\fs^{(l)}(x,x')}$ we denote the expectation value with respect to
\begin{equation}
(u,u')\sim \N\left(0,\begin{pmatrix}
    \fs^{(l)}(x,x) & \fs^{(l)}(x,x')\\
    \fs^{(l)}(x',x) & \fs^{(l)}(x',x')
\end{pmatrix}\right)\,.
\end{equation}
\begin{remark}
    To simplify the notation, we will normalize $c_l=1$ from here on.
\end{remark}
We introduce the label-smoothed cross entropy loss, introduced by \citet{szegedy2015rethinkinginceptionarchitecturecomputer}, to avoid the degenerate Hessian of the one-hot cross-entropy (the derivatives from the non ground-truth classes are 0 in this case). The dependency on the smoothing parameter $\alpha$ will drop out at the end of our calculation.
\begin{lemma} \label{lem:crossentropy}
Consider the label-smoothed cross entropy loss
\begin{align}
\mathcal{L}_\alpha(\theta) 
=& (1-\alpha) L_{\mathrm{CrossEntropy}}(\theta) +\alpha \int_X\sum_y -\frac{1}{n_Y}\log p_\theta(y|x)\dd x
\end{align}
with smoothing parameter $\alpha \in (0,1)$. Define
\begin{equation}
    y_x^\alpha =(1- \alpha)e_y + \frac{\alpha}{n_Y}\mathbf{1}\,,
\end{equation}
where $\mathbf{1}$ is the vector with all components one. Assume that the smoothed labels approximate the data well, i.e.\ $p(x) \approx y_x^{\alpha}$ and so $\epsilon(y,x)\approx0$ for all $x$, thus $H\approx\mathcal{G}=\int_XJ^\top(x)\Lambda(x)J(x)\dd x$ by Eq.~\eqref{eq:Hessianform}. Then:
\begin{equation}\lim_{m\to\infty}\frac{\Tr (H)}{\Tr(\fk)} = \frac{\alpha(2-\alpha)(n_Y-1)}{n_Y}=:\tau_{\alpha}\,.
\end{equation}
\end{lemma}
\begin{proof}
By assumption,
\begin{align}
\|p(x)\|^2\approx ||y_x ^\alpha||^2
&= \left((1-\alpha) + \tfrac{\alpha}{n_Y}\right)^2 + (n_Y-1) \left(\tfrac{\alpha}{n_Y}\right)^2 = (1-\alpha)^2 + \tfrac{2\alpha(1-\alpha)}{n_Y} + \tfrac{\alpha^2}{n_Y} \\
&= 1 - 2\alpha + \alpha^2 + \tfrac{2\alpha - 2\alpha^2 + \alpha^2}{n_Y} = 1 - \tfrac{\alpha(2-\alpha)(n_Y-1)}{n_Y}.
\end{align}
Therefore
\begin{equation}
\Tr (\Lambda(x)) = 1 - \|p(x)\|^2 = \frac{\alpha(2-\alpha)(n_Y-1)}{n_Y}=\tau_{\alpha}.
\end{equation}
And
\begin{align}
\Tr (H) 
= &\int_X\Tr (J(x)^\top \Lambda(x) J(x))\dd \mu_X 
= \int_X\Tr (\Lambda(x) J(x) J(x)^\top)\dd \mu_X\\
=&\int_X\Tr (\Lambda(x) K(x,x))\dd \mu_X
\stackrel{N\rightarrow\infty}{\longrightarrow} \int_X\fk(x,x)\Tr(\Lambda(x))\dd \mu_X
= \tau_\alpha\Tr(\fk)
\end{align}
\end{proof}
We would now like to apply NTK theory to also estimate $\Tr(H_\mathcal{E})$ in the limit. To this end, we consider the \emph{equivariant subnetwork} of the full network
\begin{align}
    \N^{\mathrm{eq}}_\theta(x) = \N_{P_\mathcal{E}\theta}(x).
\end{align}
Where $P_\E$ represents the orthogonal projection from $\Lc$ onto $\E$ with respect to the standard inner product on the parameter coordinates. It takes a simple calculation to argue that at points $\theta_*\in \mathcal{E}$, $H_\mathcal{E}$ is the Hessian of the network $\N^{\mathrm{eq}}$. In order to repeat the above calculation for $H_\mathcal{E}$, we hence need to study the NTK function $\fk_\E$ of the network $\N^{\mathrm{eq}}$, which is the focus of the next section.

\subsubsection{NTK for an equivariant subnetwork}
\label{sec:ntk_eq}
We begin with a detailed discussion on the irrep decomposition introduced in~\eqref{eq:schur-decomp}.

\textbf{Standard basis.} We work in the setting discussed above: Intermediate pre-activations $x^l$ and post-activations $h^l$ live in $V_l\cong\mathbb R^{n_l}$, carrying a representation $\rho^l$ of $G$. The network's actual coordinates — the ones in which the nonlinearity $\sigma$ is applied component-wise, and in which the weight matrices $W^l$ have i.i.d.\ Gaussian entries before any equivariance constraint — are the \emph{standard basis} coordinates, which we write $x^l\in\mathbb R^{n_l}$. This standard basis is generally not adapted to $\rho^l$.

\textbf{Change of basis.} Fix an orthogonal matrix $P^l\in O(n_l)$ that block-diagonalizes $\rho^l$ into its irrep decomposition:
\begin{equation}
    \big(P^l\big)^\top\,\rho^l(g)\,P^l = \bigoplus_\lambda \rho_\lambda(g)^{\oplus m_\lambda^l} \qquad \text{for all } g\in G,
\end{equation}
so that $V_l=\bigoplus_\lambda\rho_\lambda^{\oplus m_\lambda^l}$ in the basis produced by $P^l$. We call this the \emph{irrep-adapted basis}, and define the irrep coordinates of an arbitrary vector $v$ by
\begin{equation}
    \Tilde{v}^l_{\lambda,k,a} = \big(P^l v^l\big)_{\lambda,k,a}\,,
\end{equation}
where $\lambda$ indexes the irreps, $k\in[m_\lambda^l]$ indexes the multiplicities of each irrep $\rho_\lambda$, and $a\in[d_\lambda]$, $d_\lambda=\dim\rho_\lambda$, indexes the basis inside the irrep $\rho_\lambda$ itself. Every quantity subscripted $(\lambda,k,a)$ below lives in this irrep-adapted basis. Since $P^l$ is orthogonal, this change of basis preserves inner products and hence the (isotropic) Gaussian law of the unconstrained weights.

After enlarging by $N$ times, we acquire a new index $j$, referring to the $N$ copies of the base space $\mathbb R^{n_l}$. Note that after enlarging by $N$ times, $P^l$ is applied channel-wise, i.e.\ we have
\begin{equation}
    \Tilde{v}^l_{\lambda,j,k,a}=\big(P^l v^l_{j,\cdot}\big)_{\lambda,k,a}\,.
\end{equation}

\textbf{The nonlinearity on the full space.} The activation function $\sigma:\mathbb R\to\mathbb R$ is applied entrywise \emph{only} in the standard basis: $x^{l+1}= \sigma(h^l)$ means $(x^{l+1})_i = \sigma\big((h^l)_i\big)$ for each standard coordinate $i\in[n_{l+1}]$. Passing to irrep coordinates via $P^l$, we define
\begin{equation}
    \varsigma: V^l\rightarrow V^l,\quad x^l\mapsto P^l\,\sigma\big(\big(P^l\big)^\top x^l\big)\,,
\end{equation}
i.e.\ $\varsigma$ is $\sigma$ conjugated by the change of basis. Because $P^l$ mixes standard coordinates across different $(\lambda,k,a)$, $\varsigma$ is in general not a component-wise map on $V_l$. However, by the block structure of $P^l$ across channels, $\varsigma$ is still defined channel-by-channel on each copy of the base space $V^l$, even after we extend it by $N$ times: $\varsigma$ acting on the full layer decomposes as $N$ independent copies of the same base map, so the post-activations $\tix^{l+1}_{\lambda,j,k,a}=\varsigma(\tih^l_{j,\cdot})_{\lambda,k,a}$ depend only on channel $j$'s own pre-activation and never on any other channel $j'\neq j$. This allows us to treat the $N$ channels as conditionally i.i.d.\ throughout.

With this notation in place, the non-equivariant network's irrep coordinates across consecutive layers are related through
\begin{align}
    \tih^0_{\lambda,j,k,a} &= \frac{1}{\sqrt{n^0}}\sum_{\lambda', k',a'}\Tilde{W}^0_{\lambda\lambda',j,kk',aa'} \tix^0_{\lambda',k',a'} \\
    \tih^{l}_{\lambda,j,k,a} &= \frac{1}{\sqrt{n^{l}N}}\sum_{j',\lambda', k',a'}\Tilde{W}^{l}_{\lambda\lambda',jj',kk',aa'}\tix^l_{\lambda',j',k',a'}\label{eq:linrel}\\
    \tix^{l+1}_{\lambda,j,k,a} &= \varsigma(\tih^l_{j,\cdot})_{\lambda,k,a}\,,
\end{align}
where $\Tilde{W}^l_{\lambda\lambda',jj',kk',aa'}$ are $W^l$ with a change of basis, hence are still standard Gaussian matrices (entries i.i.d.\ $\mathcal N(0,1)$).

\textbf{Equivariant weights.} For the equivariant subnetwork, the weights $\Tilde{W}^l$ are further restricted to intertwiners $A^l$. Since the equivariant weights are obtained by orthogonally projecting the isotropically Gaussian weights of the non-equivariant network onto the intertwiner subspace, they are themselves isotropic Gaussians on the equivariant subspace: Within each irrep $\rho_\lambda$, Schur's lemma forces the weight to be shared across the $d_\lambda$ copies indexed by $a$, so the projection reduces to a single scalar $A^l_{\lambda,jj',kk'}\sim \N(0,1)$ per $(\lambda,j,j',k,k')$, i.i.d.\ across all indices and layers. The overall prefactor is not constrained by Schur's lemma, but instead determined by the norm-preserving orthogonal projection. With this, relation \eqref{eq:linrel} becomes, for the equivariant network,
\begin{align}
    \tih^{0}_{\lambda,j,k,a} &=\frac{1}{\sqrt{ n^0d_\lambda}} \sum_{k'} A_{\lambda,j,kk'}^0\,\tix^{0}_{\lambda,k',a}, \\
    \tih^{l}_{\lambda,j,k,a} &= \frac{1}{\sqrt{n^{l}Nd_\lambda}}\sum_{j',k'}  A_{\lambda,jj',kk'}^l\,\tix^{l}_{\lambda,j',k',a}. \label{eq:eqlinrel}
\end{align} In particular, conditioned on the previous layer, the $\tix^{l+1}_{\lambda,j,k,a}$ are identically distributed for same $\lambda$, $k$ and $a$ but different $j$; and are independent for different $j$, $\lambda$, and $k$, since they are built from disjoint sets of weights.

Analogous to the standard MLP case, as $N\to\infty$, the pre-activations in irrep-adapted bases $\tih^l$ converge to Gaussian processes with kernels
\begin{align}
    \Sigma_\E^l(x, x')_{(\lambda,j,k,a)(\lambda',j',k',a')}=\Eb\big[\tih^l_{\lambda,j,k,a}(x)\,\tih^l_{\lambda',j',k',a'}(x')\big],\label{eq:equiv-nngp}
\end{align}
and the NTK of the equivariant subnetwork converges to a deterministic limit,
\begin{align}
    K_\E^l(x, x')_{(\lambda,j,k,a)(\lambda',j',k',a')}=\Eb\big[\big(\nabla\tih^l_{\lambda,j,k,a}(x)\big)^\top\,\nabla\tih^l_{\lambda',j',k',a'}(x')\big]\,.\label{eq:equiv-ntk}
\end{align}
It is worth clarifying that $\Sigma_\E^l,K_\E^l:\mathbb R^{n_X}\times\mathbb R^{n_X}\to\mathbb R^{Nn_{l+1}\times Nn_{l+1}}$ are the matrix-valued equivariant NNGP kernel and NTK. Their ambient dimension $Nn_{l+1}\times Nn_{l+1}$ grows with $N$, so \emph{a priori} ``convergence as $N\to\infty$'' is not meaningful, since the kernels for different $N$ live in different spaces. Lemma~\ref{lem:kernelrec} resolves this. It shows that $\Sigma_\E^l$ and $K_\E^l$ are, for every $N$, entirely determined by a fixed collection of $d_\lambda\times d_\lambda$ blocks $\fs_\lambda^l,\fk_\lambda^l$, which are independent of $N$ and $j,k$.
\begin{lemma}\label{lem:kernelrec}
The kernels $\Sigma_\E^l$ and $K^l_\E$ admit the following block structure,
\begin{align}
    \big(\Sigma_\E^l\big)_{(\lambda,j,k,a),(\lambda',j',k',a')}(x,x')&=\delta_{\lambda\lambda'}\delta_{jj'}\delta_{kk'}\big(\fs^l_\lambda\big)_{aa'}(x,x')\\
    \big(K_\E^{l}\big)_{(\lambda,j,k,a),(\lambda',j',k',a')}(x,x')&=\delta_{\lambda\lambda'}\delta_{jj'}\delta_{kk'}\big(\fk^l_\lambda\big)_{aa'}(x,x'),
\end{align}
where each $\fs^l_\lambda$ and $\fk^l_\lambda$ are $d_\lambda$-dimensional matrices. Moreover, they have the following recursive formulas,
\begin{align}
    \big(\fs_\lambda^0\big)_{aa'}(x,x') &= \frac{1}{n^0d_\lambda}\sum_{\bar k} \tix_{\lambda,\bar k,a}\,\tix'_{\lambda,\bar k,a'},\\
    \big(\fs_\lambda^{l+1}\big)_{aa'}(x, x') &= \frac{m_\lambda^{l+1}}{n_{l+1}d_\lambda}\,\mathbb E_{u,u'\sim \fs_\lambda^l(x,x')}\big[\varsigma(u)_{a}\,\varsigma(u')_{a'}\big],\\
    \fk_\lambda^0(x,x') &=\fs_\lambda^0(x,x') ,\\
    \big(\fk_\lambda^{l+1}\big)_{a a'}(x,x') &= \big(\fs_\lambda^{l+1}\big)_{a a'}(x,x') {+} \frac{m^{l+1}_\lambda}{n_{l+1}d_\lambda}\,\mathbb E_{u,u'\sim \fs_\lambda^l(x,x')}\big[\varsigma'(u)_{a}\,\,\varsigma'(u')_{a'}\big]\fk_\lambda^l(x, x'),\label{eq:kernelrec}
\end{align}
where $u,u'\sim\fs^l_\lambda$ denotes $(u,u')\sim\mathcal N\!\left(0,\begin{bmatrix}\fs_\lambda^l(x,x)&\fs_\lambda^l(x,x')\\\fs_\lambda^l(x',x)&\fs_\lambda^l( x',x')\end{bmatrix}\right)$.
\end{lemma}
\begin{proof}
For $l=0$: by \eqref{eq:eqlinrel} and independence of $A^0$ from $\tix^0$,
\begin{align}
    \mathbb E\big[\tih^{0}_{\lambda,j,k,a}(x)\,\tih^{0}_{\lambda,j,k,a'}(x')\big] =& \frac{1}{n^0d_\lambda}\sum_{k',\bar k'}\mathbb E\big[A^0_{\lambda,j,kk'}A^0_{\lambda,j,k\bar k'}\big]\,\tix_{\lambda,k',a}\tix_{\lambda,\bar k',a'}' \\=& \frac{1}{n^0d_\lambda}\sum_{k'} \tix_{\lambda,k',a}\tix_{\lambda,k',a'}',
\end{align}
using $\mathbb E[A^0_{\lambda,j,kk'}A^0_{\lambda,j,k\bar k'}]=\delta_{k'\bar k'}$.
For the induction step, condition on layer $l$; by \eqref{eq:eqlinrel},
\begin{equation}
\mathbb E\big[\tih^{l+1}_{\lambda,j,k,a}\tih^{l+1}_{\lambda,j,k,a'}(x')\big] = \frac{1}{Nn_{l+1}d_\lambda}\sum_{j'}\sum_{k'} \tix^{l+1}_{\lambda,j',k',a}\,\tix^{l+1}_{\lambda,j',k',a'}( x'),
\end{equation}
after using $\mathbb E[A^l_{\lambda,jj',kk'}A^l_{\lambda,j\bar j',k\bar k'}]=\delta_{j'\bar j'}\delta_{k'\bar k'}$. Since the pairs $\big(\tix^{l}_{\lambda,j',k',\cdot}(x),\tix^{l}_{\lambda,j',k',\cdot}(x')\big)$, $j'\in[N]$, are i.i.d.\ copies of $(\varsigma(u),\varsigma(u'))$ with $u,u'\sim\fs_\lambda^l(x,x')$; the average over $j'$ converges almost surely to $\fs_\lambda^{l+1}$ by the law of large numbers as $N\to\infty$, and the sum over $k'$ collapses into $m_\lambda^l$ counts of the single expectation, since the $k'$-summands are identically distributed.

For the NTK, $K_\E^0=\Sigma_\E^0$ at layer $0$, so the block structure is automatically satisfied. For the induction step, split the tangent vector into the last-layer and earlier-layer parts,
\begin{align}
    &\big\langle \nabla_{A}\tih^{l+1}_{\lambda,j,k,a}(x),\nabla_A \tih^{l+1}_{\lambda,j,k,a'}(x')\big\rangle\\=&\big\langle \nabla_{A^{l+1}}\tih^{l+1}_{\lambda,j,k,a}(x),\nabla_{A^{l+1}} \tih^{l+1}_{\lambda,j,k,a'}(x')\big\rangle + \big\langle \nabla_{A^{<l+1}}\tih^{l+1}_{\lambda,j,k,a}(x),\nabla_{A^{<l+1}} \tih^{l+1}_{\lambda,j,k,a'}(x')\big\rangle\,.
\end{align}
The first term is computed exactly as for $\fs_\lambda^{l+1}$ above and converges to $ \big(\fs_\lambda^{l+1}\big)_{aa'}(x,x')$. For the second, the chain rule applied to \eqref{eq:eqlinrel} gives
\begin{equation}
    \partial_{A^{<l}} \tih^{l+1}_{\lambda,j,k,a} = \frac{1}{\sqrt{Nn_{l+1}d_\lambda}}\sum_{j',k'}A^{l+1}_{\lambda,jj',kk'}\,\varsigma'(\tih^{l}_{j'})_{\lambda,k',a}\,\partial_{A^{<l}}\tih^{l}_{\lambda,j',k',a},
\end{equation}
so that the second term, analogous to the NNGP kernel, converges a.s.\ by the induction hypothesis as $N\to\infty$ to
\begin{equation}
\frac{m_{\lambda}^{l+1}}{n_{l+1}}\,\mathbb E_{u,u'\sim \Sigma_\E^l(x,x')}\big[\varsigma'(u)_{a}\,K_\E^l(x,x')\,\varsigma'(u')_{a'}\big].
\end{equation}
Furthermore, since $\fk^l$ is fixed, it could be extracted from the expectation, giving \eqref{eq:kernelrec}.
\end{proof}
\begin{remark}
In particular, the NNGP and NTK kernels we calculate here can be interpreted as the NTK for a \emph{layerwise equivariant} architecture. Such a kernel has previously only been calculated in the case when all intermediate representations are \emph{regular representations of a finite group} \citep{pmlr-v267-misof25a}. We extend these results to a much more general setting.
\end{remark}
Note that there is an extra term $\frac{m^l_{\lambda}}{n_ld_\lambda}$ for the recursive formula. These terms accumulate through layers and result in a bound of trace against the traces of the ambient kernels $\Sigma^l$ and $K^l$.

\begin{lemma}\label{lem:tracebound}
Assume $\sigma$ is the elementwise ReLU. Define
\begin{equation}
\beta^l = 
    \max_\lambda \frac{m_\lambda^{l+1}}{d_\lambda\, n^{l+1}},
\end{equation}
Then for all $l\geq 0$,
\begin{equation}\label{eq:kernelineq}
   \Tr\Sigma_\E^l(x,x) \le \left(\prod_{l'=0}^{l}\beta^{l'}\right)\Tr\Sigma^l(x,x), \qquad
    \Tr \Sigma_\E^l(x,x) \le\left(\prod_{l'=0}^{l}\beta^{l'}\right)\Tr K^l(x,x)\,.
\end{equation}

\end{lemma}
\begin{proof}
\emph{GP kernel.} By Lemma~\ref{lem:kernelrec}, summing the trace of $\Sigma_\E^0$ over the $Nm_\lambda^0$ copies of each irrep, we have,
\begin{equation}
\Tr\Sigma_\E^0(x,x) = \frac{1}{n^0}\sum_\lambda \frac{Nm_\lambda^1}{d_\lambda}\sum_{k',a}\tix_{\lambda,k',a}^2 \le  \beta^0 \frac{Nn_1}{n_0}\|\tix\|^2 = \beta^0 Nn_1\,\fs^0(x,x)=
\beta^0\Tr\Sigma^0(x,x).
\end{equation}
The second to last equality is because the change of basis preserves the norm, and the last equality is because $\Sigma^0(x,x')$ is a $Nn_1\times Nn_1$ dimensional matrix. 

For the induction step, homogeneity of $\sigma$ (ReLU) gives 
\begin{align}
\Eb_{u\sim\fs^l_\lambda}[\|\varsigma(u)\|^2]=&\Eb_{u\sim\Sigma^l_E}[\|\sigma(\hat Pu)\|^2]= \mathbb E_{\chi\sim \N(0,1)}[\sigma(\chi)^2]\Tr\big[P^l\fs^l_\lambda(x,x)\big(P^l\big)^\top\big]\\
=&\mathbb E_{\chi\sim \N(0,1)}[\sigma(\chi)^2]\Tr\fs^l_\lambda(x,x),\\
\Eb_{u\sim\fs^l}[\sigma(u)^2]=&E_{\chi\sim \N(0,1)}[\sigma(\chi)^2]\fs^l(x,x).
\end{align}
For notational convenience, we denote $\mathbb E_{\chi\sim \N(0,1)}[\sigma(\chi)^2]$ as $k_\sigma$, so that
\begin{align}
    \Tr\Sigma^{l+1}(x,x)=& n_{l+2}\fs^{l+1}(x,x)=n_{l+2}\Eb_{u\sim \fs^l}[\sigma(u)^2]\\
    =&n_{l+2}k_\sigma\fs^l(x,x)=\frac{n_{l+2}}{n_{l+1}}k_\sigma \Tr\Sigma^l(x,x).
\end{align}
By induction, we have
\begin{align}
    \Tr \Sigma^{l+1}_\E(x,x)=&\sum_{\lambda,j}\frac{m_\lambda^{l+1}}{n_{l+1}d_\lambda}\sum_{k=1}^{m_\lambda^{l+2}}\Eb_{u\sim\fs_\lambda^l}[\|\varsigma(u)\|^2]=\sum_{\lambda}\frac{Nm_\lambda^{l+2}}{n_{l+1}d_\lambda}m_\lambda^{l+1}\Eb_{u\sim\fs_\lambda^l}[\|\varsigma(u)\|^2]\\
    \leq& \beta^{l+1} \frac{n_{l+2}}{n_{l+1}}\sum_\lambda Nm^{l+1}_\lambda k_\sigma\Tr\fs_\lambda^l(x,x)=\beta^{l+1} \frac{n_{l+2}}{n_{l+1}}k_\sigma\Tr\Sigma_\E^l(x,x)\\
    \leq& \left(\prod_{l'=0}^{l+1}\beta^{l'}\right)\frac{n_{l+2}}{n_{l+1}}k_\sigma\Tr\Sigma^l(x,x)=\left(\prod_{l'=0}^{l+1}\beta^{l'}\right)\Tr\Sigma^{l+1}(x,x)
\end{align}

\emph{NTK.} For ReLU, $\sigma'$ is the indicator of $\mathbb R_{>0}$, so for $u\sim\mathcal N(0,\Sigma_\E^l(x,x))$, $\mathbb E[\sigma'(u)^2]=\mathbb P(u\ge0)=\tfrac12$, independent of the variance of $u$. A calculation analogous to the GP bound gives the desired bound for NTK.
\end{proof}

Since $G$ acts trivially on the output layer, $\rho_\lambda=\rho_{\mathrm{triv}}$ there, $d_\lambda=1$, and only $a=a'$ occurs in \eqref{eq:kernelrec}. Hence $\fk_\lambda^{L}(x,x')$ is a scalar function, and the matrix-valued $K_\E^{L}(x,x')$ collapses to $\fk_\lambda^{L}(x,x')$ times the identity $\mathrm{Id}_{n_Y}$, and the bound applies directly to the scalar kernel.

\begin{corollary}
$K_\E^L(x,x') = \fk_\E^L(x,x')\,I_{n_Y}$, and
\begin{equation}
\fk_\E^L(x,x) \le \Big(\prod_{l=1}^L \beta^l\Big)\fk^L(x,x),
\end{equation}
Note that for the output layer, $\beta^{L+1}=1$, hence $\beta^{L+1}$ does not contribute to the coefficient.

\end{corollary}

Combining Lemma~\ref{lem:tracebound} with the identity relating $\operatorname{tr}H_{\mathcal E}$ to $K_\E^L$ used in the proof of Lemma~\ref{lem:crossentropy} then gives, exactly as before,
\begin{lemma}\label{lemma:trratio}
Under the assumptions above, for the cross-entropy loss and near an equivariant minimizer, the spectral proportion of $H_{\mathcal E}$ relative to $H$ is at most $\prod_{l=1}^L \beta^l$, with $\beta^l=\max_\lambda m_\lambda^l/(d_\lambda n^l)$.
\end{lemma}
\subsection{Main theoretical result}
Plugging the results of Lemma~\ref{lemma:trratio} and~\eqref{eq:dimratio} into~\eqref{eq:inf-equiv-improv-ratio} leads to our main theorem. For convenience we list here the assumptions we made during its derivation:
\begin{itemize}
    \item We average over long times, i.e.\ $T$ large.
    \item The neural networks are wide, i.e.\ $N$ large.
    \item The group $G$ is finite, and  all intermediate spaces decompose into irreps of real type.
    \item The model is well trained, i.e.\ the error $\epsilon$ is small.
    \item The non-linearity is the ReLU function. 
\end{itemize}
Under these conditions, we have
\begin{theorem}
\label{thm:improvement-ratio-general}
    The asymptotic relative ratio of improvement for the equivariance of the solution satisfies the bound
    \begin{equation}
        \lim_{T\rightarrow \infty}\frac{R_{\E^\perp}(T)}{R(T)} \geq \frac{r|G|}{r|G|-1}\left(1-\prod_{l=1}^{L}\max_\lambda \frac{m_\lambda^l}{d_\lambda n^l}\right)\,
        \label{eq:improv-rat-bound},
    \end{equation}
    In particular, the asymptotic relative ratio is bigger than one if
    \begin{align}
        \frac{1}{r|G|} >\prod_{l=1}^L \max_\lambda \frac{m_\lambda^l}{d_\lambda n^l}. \label{eq:improv-cond}
    \end{align}
\end{theorem}
As an important special case, consider rectangular networks, i.e. $m_\lambda^{l}=m_\lambda$ (consequently, $r=1$) and $n_l=n$ for all $\lambda$ and $1\leq l\leq L-1$. In this case, Theorem~\ref{thm:improvement-ratio-general} simplifies to
\begin{corollary}
    \label{cor:rec-net-bound}
    For rectangular networks, the asymptotic relative ratio of improvement for the equivariance of the solution satisfies the bound
    \begin{equation}
        \lim_{T\rightarrow \infty}\frac{R_{\E^\perp}(T)}{R(T)} \geq \frac{|G|}{|G|-1}\left(1-\left(\max_\lambda \frac{m_\lambda}{d_\lambda n}\right)^{L-1}\right)\,,
        \label{eq:rec-net-bound}
    \end{equation}
    which is larger than one if 
    \begin{equation}
         \frac{1}{|G|} >\left(\max_\lambda \frac{m_\lambda}{d_\lambda n}\right)^L
    \end{equation}
\end{corollary}
\begin{remark}
    The bound in Theorem~\ref{thm:improvement-ratio-general} explicitly depends on the representation $\bar\rho$ on $\mathcal{H}$ via the multiplicities $m_\lambda^l$. This may seem unintuitive since there is considerable freedom in the choice of the $\bar\rho$. However, the bounded improvement ratio also depends on $\bar\rho$ via the equivariant minimum around which the loss is expanded. Hence, the equivariant minimum to which the network converged at the end of training selects the $m_\lambda^l$ in~\eqref{eq:improv-rat-bound}.
\end{remark}

Note the significance of the asymptotic ratio being greater than one -- in this case $R_{\E^\perp}>R$ in the limit of $T\rightarrow\infty$ and hence the expected equivariance improvement due to SWA (as measured by $R_{\E^\perp}$) outpaces the expected equivariance improvement due to SWA (as measured by $R$), under data augmentation. Some remarks are in order.
First,  note 
that for a given architecture, the condition \eqref{eq:improv-cond} is harder to satisfy if the group is larger. This aligns with the intuition that tasks with larger symmetries are harder to learn, see for instance the discussion in~\cite{NEURIPS2025_d96fcc07}. Secondly, for any intermediate representation, we have
\begin{equation}
    n^l = \sum_{\lambda} m_\lambda^ld_\lambda \geq \max_{\lambda} m^l_\lambda d_\lambda \geq (\inf_\lambda d_\lambda^2) \left(\max_\lambda \frac{m_\lambda^l}{d_\lambda}\right) \, \Rightarrow \, \max_\lambda \frac{m_\lambda^l}{n^ld_\lambda}\leq \frac{1}{\inf_\lambda d_\lambda^2}\leq 1\,.
\end{equation}
Hence, as long as sufficiently many intermediate representations are either mixtures of irreps or irreps of non-trivial dimension, the condition~\eqref{eq:improv-cond} will be satisfied. This aligns with the observation that in a manifestly symmetric network, using equivariant features up until the final layer is beneficial even if the task considered is invariant. Finally, we note that in many cases, $\max_\lambda \frac{m_\lambda^l}{n^ld_\lambda}$ is significantly smaller than $1$. For example, if an intermediate representation is regular (as in a group-CNN \citep{cohen2016groupequivariantconvolutionalnetworks}), we have $m_\lambda^l=d_\lambda$ for all $\lambda$  \citep[Cor. 2.18]{fulton1991}, so that $\max_\lambda \frac{m_\lambda^l}{n^ld_\lambda}= \frac{1}{n^l}=\frac{1}{|G|}$ for that $l$. 

\begin{remark}
    The results and conclusions above hold also in the case of irreps of non-real type, with an adjusted definition of $\beta^l$:
    \begin{align}
        \beta^l =\max_\lambda \frac{m^{l+1}e_\lambda}{n^{l+1}d_\lambda},
    \end{align}
    where $e_\lambda = \dim E_\lambda\leq d_\lambda$. In particular, for regular representations, since $m_\lambda=d_\lambda/e_\lambda$, we still have $\beta^l = \frac{1}{|G|}$. The proof requires only minor adjustments, which add little to the understanding while considerably complicating the argument; we therefore refrain from presenting them.

\end{remark}

\subsection{A proxy for the orthogonal loss}
Before moving on to the experiments, we give an auxiliary result that justifies the equivariance metric we use in practice.

Recall that we restrict the scope to group invariant classification tasks. To achieve exact equivariance, a direct method is to take the group average of an output,
\begin{equation}
     \N^G_\theta(x) = \frac{1}{|G|}\sum_g\N_\theta(g_Xx)
\end{equation}
Note that $\bar \N_\theta$ is different from $\N_\theta^\text{eq}$, where the average is taken in the parameter space instead of the output space. Under this setting, a natural metric for the equivariance of the network is the Kullback-Leibler divergence between the logits given by the raw output $\N_\theta$ and its group-averaged version $ \N^G_\theta$:
\begin{equation}
    \label{eq:Leq-def}
     L_{\mathrm{eq}}(\theta) = \int_{V_X}D_{\mathrm{KL}}\left( p_\theta^G(x)\,\middle\|\,  p_\theta(x)\right)\dd \mu_X,\
    \text{where } p_\theta^G(x) = \operatorname{softmax}\left( \N^G_\theta(x)\right).\\
\end{equation}
We now show that it is an accurate proxy for the orthogonal loss under our assumptions.
\begin{proposition}
\label{prop:lperp}
Under the assumption that $\varepsilon(x) = p_\theta(x) - e_{y(x)}$ is small, for $\theta$ in a neighbourhood of the equivariant minimizer $\theta_*=0$,
\begin{equation}
  L_{\mathrm{eq}}(\theta)
  \;=\;
   L_{\E^\perp}(\theta) + O\!\left(\|\theta\|^{3}\right) + O\!\left(\varepsilon\|\theta\|^{2}\right).
\end{equation}
In particular, the equivariant component $\theta_{\E}$ does not contribute at this order.
\end{proposition}
\begin{proof} 
Let $J(x) = \nabla_\theta \N_{\theta}(x)\big|_{\theta =0}$ be the Jacobian at the minimizer. Linearizing the network output gives $\N_\theta(x) = \N_{0}(x) + J(x)\theta + O(\|\theta\|^2)$. Since the minimizer $\theta^*=0$ is equivariant, its Jacobian satisfies $J(g_X x) = J(x)\bar{g}^{\top}$, which gives
\begin{equation}
    \N_\theta^G(x) = \N_{0}(x) + \frac{1}{|G|}\sum_gJ(x)\bar g^\top\theta + O(\|\theta\|^2) = \N_{0}(x) + J(x)\theta^\E + O(\|\theta\|^2) \,.
\end{equation}
Taking the difference, we obtain
\begin{equation}
  \delta(x) \;=\; \N_\theta(x) -  \N^G_\theta(x)
  \;=\; J(x)\,\theta^{\E^\perp} + O\!\left(\|\theta\|^2\right).
  \label{eq:u-is-perp}
\end{equation}
Next, analogous to Eq.~\eqref{eq:Hessianform}, we expand the Kullback--Leibler divergence to second order,
\begin{align}
  D_{\mathrm{KL}}\big( p^G_\theta(\cdot\mid g_X x) \,\|\, p_\theta(\cdot\mid x)\big)
  &\;=\; \tfrac{1}{2} \delta(x)^{\top} \Lambda(x)\, \delta(x) + O\!\left(\|\delta(x)\|^{3}\right) \notag \\
  &\;=\; \tfrac{1}{2}\big(\theta^{\E^\perp}\big)^{\!\top} \, J(x)^{\top}\Lambda(x) J(x)\, \theta^{\E^\perp} + O\!\left(\|\theta\|^{3}\right).
\end{align}
Taking the expectation over $x \sim\mu_X$ yields $L_{\mathrm{eq}}(\theta)$.
\begin{equation}
  L_{\mathrm{eq}}(\theta)
  \;=\;
  \tfrac{1}{2}\big(\theta^{\E^\perp}\big)^{\!\top} \mathcal{G}\, \theta^{\E^\perp} + O\!\left(\|\theta\|^{3}\right).
\end{equation}
Where $\mathcal{G}$ is the Gauss-Newton matrix in~\eqref{eq:Hessianform}. Since the full Hessian is $H = \mathcal{G} + R$ with the residual $R$ being linear in $\varepsilon$, we can replace $\mathcal{G}$ with $H$ at the cost of an $O(\varepsilon\|\theta\|^2)$ error. We conclude that
\begin{equation}
    L_{\mathrm{eq}}(\theta) \;=\; \tfrac{1}{2}\big(\theta^{\E^\perp}\big)^{\top} H_{\E^\perp}\, \theta^{\E^\perp} + O\!\left(\|\theta\|^3\right) + O\!\left(\varepsilon\|\theta\|^2\right) \;=\; L_{\E^\perp}(\theta) + O\!\left(\|\theta\|^3\right) + O\!\left(\varepsilon\|\theta\|^2\right).\qedheresafe
\end{equation}
\end{proof}
\begin{corollary}
\label{cor:ratio-estimable}
The ratio of equivariance losses evaluated on unaveraged and averaged weights satisfies
\begin{equation}
  \frac{\mathbb{E}\big[L_{\mathrm{eq}}(\theta_T)\big]}
       {\mathbb{E}\big[L_{\mathrm{eq}}(\bar\theta_T)\big]}
  \;=\; R_{\E^\perp}(T)\,\big(1 + o(1)\big)
\end{equation}
in the regime in which Proposition~\ref{prop:lperp} applies.
\end{corollary}

\begin{remark}
From a practical standpoint, the orthogonal loss $L_{\E^\perp}$ is a theoretical construct: computing it explicitly requires knowledge of the Hessian and the orthogonal projection in the parameter space, which is typically intractable for deep networks. In contrast, $L_{\mathrm{eq}}$ is an easily computable, well-motivated metric that directly measures the equivariance of the network's predictions. Proposition~\ref{prop:lperp} essentially bridges this gap, establishing $L_{\mathrm{eq}}$ as a proxy for $L_{\E^\perp}$ near a well-fitted minimum. Consequently, any theoretical predictions or dynamics derived for $L_{\E^\perp}$, in particular, the ratio $R_{\E^\perp}$, can directly transit to $L_{\mathrm{eq}}$. This justifies the use of $L_{\mathrm{eq}}$ in our experiments to evaluate the learned equivariance of the network, with the confidence that this empirical metric follows the theoretical behavior of $L_{\E^\perp}$.
\end{remark}


\section{Experiments}
In this section, we present both synthetic and practical experiments to validate our theoretical results. In Section~\ref{sec:synth-expmts}, we conduct synthetic experiments to verify the theoretical predictions of Theorem~\ref{thm:improvement-ratio-general}. In Section~\ref{sec:practical}, we evaluate the performance of SWA on a range of image and graph classification tasks, chosen to cover diverse data modalities and architectural families with a chosen augmentation group.

\subsection{Synthetic experiments}
\label{sec:synth-expmts}
We conduct synthetic experiments to validate our theoretical settings, and verify our theoretical results. Throughout the section, we work with multi-layer perceptrons (MLPs) since they have completely unrestricted parameter space, and our theory applies without modification.

\subsubsection{Empirical validation of the quadratic approximation}
\label{sec:quad}
Before applying our theoretical framework to weight averaging, we first conduct an experiment to validate the quadratic approximation of the loss landscape near a well-fitted minimum. We employ a teacher-student setup where both models are MLPs with identical architectures, and the input data is sampled from a standard Gaussian distribution. In this controlled setting, the global minimum is explicitly known, allowing us to compute the exact empirical Hessian at the minimizer. 

To simulate the dynamics of late-stage training, we perturb the parameters away from the minimum with stochastic gradient ascent, which explicitly pushes the parameters along the noise distribution of SGD. We then perform standard stochastic gradient descent from these perturbed points, tracking both the true loss and its theoretically estimated quadratic approximation.

We first analyze the spectral properties of the Hessian. As illustrated in Figure~\ref{fig:hessrank}, the effective rank of the Hessian monotonically increases with the number of data samples used for its computation, eventually plateauing at a sample size of 45. Furthermore, the resulting matrix is nearly full rank; the difference is primarily due to the positive homogeneity of the ReLU activation functions, which causes scale invariances in the parameter space.

Despite this slight rank deficiency, the quadratic approximation remains accurate under the training dynamics. Figure~\ref{fig:quadapprox} compares the true loss against the quadratic estimation for the optimization trajectory originating from the gradient ascent perturbations. The approximation aligns closely with the full loss, and the approximation error strictly decreases and stabilizes as the parameters converge back toward the minimizer. Crucially, this high accuracy holds because the parameter trajectory is not exploring the landscape arbitrarily. As is shown with the calculation, the covariance matrix of the SGD noise is asymptotically the Hessian, causing the optimization path to be heavily biased toward the top eigenspaces where the Hessian dominates the local geometry. By naturally avoiding flat or highly non-linear directions, the SGD trajectory remains confined to a region where the loss is strictly dominated by the quadratic term, rendering our approximation highly accurate.
\begin{figure}
    \centering
    \begin{subfigure}{0.33\linewidth}
        \centering
        \includegraphics[width=\linewidth]{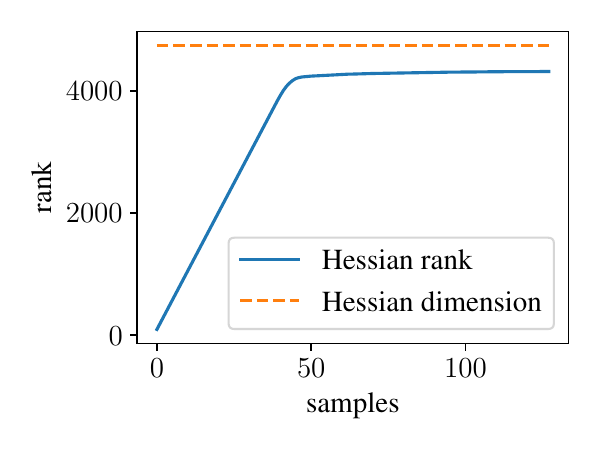}
        \caption{}
        \label{fig:hessrank}
    \end{subfigure}
    \begin{subfigure}{0.66\linewidth}
        \centering
        \includegraphics[width=\linewidth]{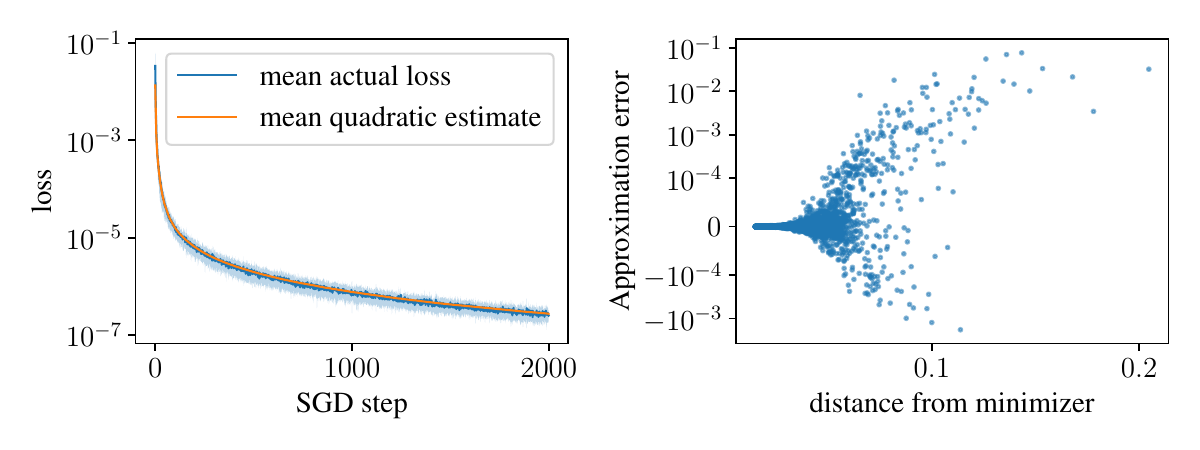}
        \caption{}
        \label{fig:quadapprox}
    \end{subfigure}
    \caption{Empirical validation of the quadratic approximation in a teacher-student MLP setting. 
        \textbf{(a)} Effective rank of the empirical Hessian evaluated at the global minimum, plotted against the number of data samples used.
        \textbf{(b)} Evaluation of the quadratic approximation along an optimization trajectory (sampled via gradient ascent). \textit{Left:} The true loss strictly follows the quadratic estimation. \textit{Right:} The approximation error between full- and quadratic loss stabilizes and strictly decreases as the parameter distance to the minimizer diminishes.}
    \label{fig:example}
    \vspace{-1em}
\end{figure}
\subsubsection{Verification of the theorem}
\label{sec:verification}
With the quadratic approximation for the loss function justified, we now extend the teacher-student setup of Section~\ref{sec:quad}, and verify our main theorem (Theorem~\ref{thm:improvement-ratio-general}).

\textbf{Setup.}
In this experiment, the teacher and student share the same expressivity in terms
of equivariant functions, but they are no longer identical: the teacher is a
multilayer perceptron whose hidden layers carry the regular representation of
$G = C_n$, so that the network is exactly $G$-equivariant by construction. On the other hand, the student is an
ordinary MLP of matching width and depth, with no constraint tying its
weights together. Working with MLPs rather than convolutional architectures is
deliberate: the student's parameter space is then the full space with no affine restriction $\mathcal{L} \subsetneqq
\mathcal{H}$, which is exactly the setting in which Theorem~\ref{thm:improvement-ratio-general} was
derived and the theorem applies without modification.

The teacher's weights are embedded into the student's parameter space. Therefore, the
minimizer is accessible to the student without relying on the convergence of SGD
to recover it. Because the layerwise representations are explicit, the group
action $\bar\rho$ on the parameter space is a permutation of coordinates, the decomposition $\mathcal{E} \oplus \mathcal{E}^{\perp}$ is
explicit, and the structural factors $r$ and $m_\lambda$ are directly accessible. As a result, the trace ratio
$\operatorname{Tr}(H_{\mathcal{E}^{\perp}})/\operatorname{Tr}(H)$ and the
asymptotic ratios are evaluated exactly.

Starting from $\theta_*$, we perturb the student away from $\theta_*$ by entrywise independent Gaussian diffusion added
directly to the weights, with the standard deviation proportional to that of each weight. From each sampled parameter, we
train the network with SGD, 
evaluating $L(\theta_T)$, $L(\bar\theta_T)$, $L_{\mathrm{eq}}(\theta_T)$ and
$L_{\mathrm{eq}}(\bar\theta_T)$ along the trajectory, from which we obtain
empirical estimates of $R(T)$, $R_{\mathrm{eq}}(T)$ and their ratio.

\textbf{Dependence on the averaging time.}
For a fixed cyclic group $G = C_4$ acting on the input, we track
$R_{\mathrm{eq}}(T)/R(T)$ as a function of the number of SWA epochs $T$; the result
is shown in Figure~\ref{fig:synthetic_ratio}~(left). The ratio increases
monotonically from values below the asymptotic bound early in averaging, before
gradually decreasing and plateauing above
the theoretical bound for large $T$.

\textbf{Dependence on the group size.}
We repeat the same procedure for teachers equivariant under cyclic groups $C_n$
of increasing order, holding the student architecture and training budget
fixed, and record the plateaued value of $R_{\mathrm{eq}}(T)/R(T)$ for each $n$; see
Figure~\ref{fig:synthetic_ratio}~(right). The resulting values decrease with
$|G|$, in qualitative agreement with the trend implied by
Theorem~\ref{thm:improvement-ratio-general} and the discussion following it: larger symmetry groups lead to lower dimension ratio, since $r|G|/(r|G|-1) \to 1$ as $|G|$ grows, so the guaranteed
excess of $R_{\mathrm{eq}}$ over $R$ shrinks correspondingly. Within the range of
group sizes tested, the empirical ratio nonetheless remains bounded above $1$
throughout, indicating that the equivariance-specific benefit of SWA persists
even as it weakens for larger groups.

Taken together, these two experiments substantiate both qualitative predictions
of Theorem~\ref{thm:improvement-ratio-general}: convergence of $R_{\mathrm{eq}}/R$ toward its asymptotic
value as the averaging window grows, and a decreasing dependence of that
asymptotic value on the size of the augmentation group.

\begin{figure}[t]
    \centering
    \begin{subfigure}[b]{0.45\textwidth}
        \centering
        \includegraphics[width=\textwidth]{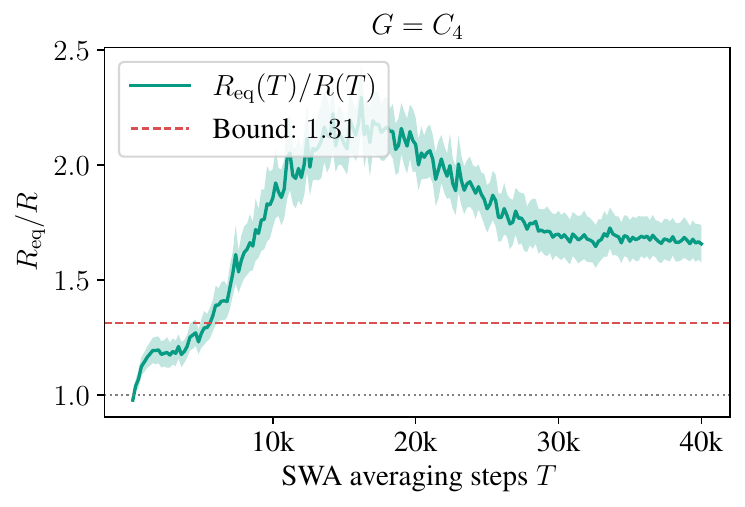}
        \caption{}
        \label{fig:ratio_vs_T}
    \end{subfigure}
    \hfill
    \begin{subfigure}[b]{0.45\textwidth}
        \centering
        \includegraphics[width=\textwidth]{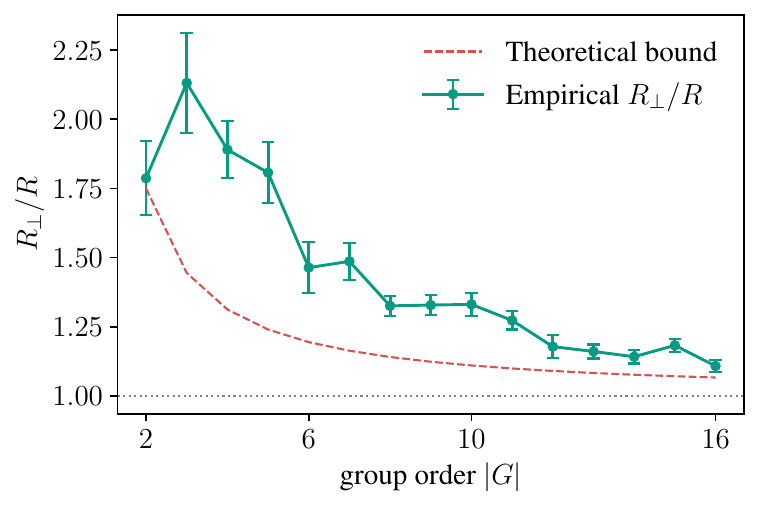}
        \caption{}
        \label{fig:ratio_vs_group}
    \end{subfigure}
    \caption{Synthetic verification of Theorem~\ref{thm:improvement-ratio-general} in the
   teacher-student setting. \textbf{(a)}~The empirical ratio
    $R_{\mathrm{eq}}(T)/R(T)$ for $G = C_4$ as a function of the SWA averaging time
    $T$. The curve gradually rises above the asymptotic bound
    $r|G|/(r|G|-1)\cdot 1/(1-|G|^{-3}) \approx 1.31$ (dashed line) and plateaus after mild decrease.
    \textbf{(b)}~Plateaued value of $R_{\mathrm{eq}}/R$ against the group order $|G|$ for cyclic groups $C_n$. }
    \label{fig:synthetic_ratio}
    \vspace{-1em}
\end{figure}

\subsection{Practical tasks}
\label{sec:practical}
We evaluate the performance of SWA on a range of image and graph classification tasks, chosen
to cover diverse input modalities and architectural families while
keeping the augmentation group fixed across each experimental block.
\subsubsection{Datasets and architectures}

\textbf{Image classification.}
We use [MNIST, Fashion-MNIST, CIFAR-10, CIFAR-100, and ImageNet-100],
augmented under the action of discrete cyclic subgroups of $SO(2)$. For the primary experiments, we use the group of 90-degree rotations $C_4$. This range spans low-resolution grayscale images, medium-resolution natural images with varying class counts (CIFAR-10/100), and higher-resolution natural images (ImageNet-100).

To explore the generality across architectures of the equivariance gain from SWA, we evaluate four families on this image-classification task:
\begin{itemize}
  \item \textbf{Multilayer perceptrons (MLPs).} These serve as a baseline and provide the closest match to the theoretical setting of Section~\ref{sec:theory}, where the parameter space is unrestricted, and Theorem~\ref{thm:improvement-ratio-general} applies without modification.
  \item \textbf{Convolutional networks (CNNs)}: VGG-16, PreResNet-164, and WideResNet-28$\times$10. CNNs are the standard architecture for image classification and admit a natural lift to their group-equivariant    counterparts (G-CNNs), which are exactly $G$-equivariant by construction. Comparing standard CNNs trained with $G$-augmentation against this architectural baseline isolates the contribution of SWA-induced statistical equivariance from architecturally-enforced equivariance.
  \item \textbf{Vision transformers}: ViT-S. ViTs process inputs as patch sequences, so any equivariance present is learned from data rather than architecturally imposed. This makes ViTs a test of whether SWA improves equivariance as a generic sampling mechanism, rather than a consequence of the architectural inductive biases of CNNs.
  \item \textbf{Graph neural networks on superpixels}: ViG-Ti applied to superpixel graphs extracted from the image inputs. This architecture lies between CNNs and pure GNNs. The input geometry is derived from the image, but the computation over graphs ignores the pixel grid information. It tests whether the SWA equivariance gain persists when the data geometry is diluted into node and edge features from exact pixel values.
\end{itemize}
\textbf{Graph classification.}
To verify that the framework extends beyond image data, we additionally evaluate on three graph classification benchmarks: [DHFR, BZR, COX2], using graph isomorphism networks. Here the relevant symmetry is 3D rotation, and we augment with discrete subgroups of $SO(3)$. We use the GIN introduced by \cite{xu2019powerfulgraphneuralnetworks}; it takes in positional features of the atoms and is therefore not manifestly equivariant. The graph classification experiments serve as a confirmation of consistency of the SWA effect ($\Delta L_{\mathrm{eq}} < 0$).
\subsubsection{Training and evaluation}
For training, we use cross entropy loss and the stochastic gradient descent (SGD) optimizer, along with canonical choices of hyperparameters for each architecture. For each model, we train for a number of epochs until convergence, and then train for an additional number of epochs while taking the average of the weights at the end of each epoch. We evaluate the models before and after applying SWA, and the gain from applying SWA.

Besides the cross entropy loss and accuracy on the test set, we use the following metrics for evaluation of equivariance:
\begin{itemize}
  \item Orbit Same Prediction (OSP): the average proportion of the prediction within the same orbit that matches the mode prediction of the orbit. This metric is a direct measure of the equivariance of the model, as it quantifies how consistent the predictions are within the same orbit. Mathematically, the OSP can be defined as:
\begin{equation}
\mathrm{OSP} = \frac{1}{N} \sum_{i=1}^{N}\frac{1}{|G|}\max_{1\leq j \leq n_Y} \sum_{g\in G} \mathbf{1}\{\hat{y}_{g_X x_i} = j\}
\end{equation}
where $N$ is the number of test samples, and $\hat{y}_x$ denotes the prediction of sample $x$. Note that the OSP is lower bounded by the accuracy, and for a perfectly equivariant model, the OSP would be $100\%$.
\item Equivariance Loss: the KL-divergence between the prediction of a sample and the average prediction of its orbit, as defined in Eq.~\eqref{eq:Leq-def}. This serves as a continuous measure that captures how much the predictions deviate from being equivariant.
\end{itemize}

\subsubsection{Results}
We begin with image classification results, reporting the absolute gains in accuracy and equivariance from SWA, and show that the relative improvement ratio $R_{\mathrm{eq}}/R$ is consistently larger than $1$. That is, SWA improves equivariance more than it improves general performance.

\begin{table*}[t]
\centering
\footnotesize
\setlength{\tabcolsep}{4pt}
\begin{tabular}{ll
    S[table-format=2.2]@{\,}l
    S[table-format=2.2]@{\,}l
    S[table-format=-1.2]@{\,}l
    S[table-format=2.2]@{\,}l
    S[table-format=2.2]@{\,}l
    S[table-format=-1.2]@{\,}l
}
\toprule
& & \multicolumn{6}{c}{Accuracy (\%)} & \multicolumn{6}{c}{OSP (\%)} \\
\cmidrule(lr){3-8} \cmidrule(lr){9-14}
Dataset & Model 
  & \multicolumn{2}{c}{Before} & \multicolumn{2}{c}{After} & \multicolumn{2}{c}{$\Delta$} 
  & \multicolumn{2}{c}{Before} & \multicolumn{2}{c}{After} & \multicolumn{2}{c}{$\Delta$} \\
\midrule
\multirow{2}{*}{MNIST} 
 & MLP & 94.55 & \scriptsize$\pm$0.20 & 94.68 & \scriptsize$\pm$0.15 & 0.13 & \scriptsize$\pm$0.11 & 97.73 & \scriptsize$\pm$0.12 & 97.93 & \scriptsize$\pm$0.09 & 0.20 & \scriptsize$\pm$0.07 \\
 & CNN & 74.23 & \scriptsize$\pm$6.59 & 71.74 & \scriptsize$\pm$3.91 & -2.49 & \scriptsize$\pm$6.70 & 94.56 & \scriptsize$\pm$0.57 & 96.85 & \scriptsize$\pm$0.24 & 2.30 & \scriptsize$\pm$0.56 \\
\midrule
\multirow{2}{*}{FashionMNIST} 
 & MLP & 85.28 & \scriptsize$\pm$0.09 & 85.98 & \scriptsize$\pm$0.07 & 0.71 & \scriptsize$\pm$0.11 & 95.62 & \scriptsize$\pm$0.21 & 96.71 & \scriptsize$\pm$0.16 & 1.09 & \scriptsize$\pm$0.15 \\
 & CNN & 75.43 & \scriptsize$\pm$2.79 & 85.21 & \scriptsize$\pm$0.08 & 9.77 & \scriptsize$\pm$2.74 & 96.08 & \scriptsize$\pm$0.33 & 98.45 & \scriptsize$\pm$0.06 & 2.37 & \scriptsize$\pm$0.30 \\
\midrule
\multirow{3}{*}{CIFAR-10} 
 & VGG-16 & 87.83 & \scriptsize$\pm$0.16 & 89.83 & \scriptsize$\pm$0.09 & 2.00 & \scriptsize$\pm$0.20 & 92.97 & \scriptsize$\pm$0.15 & 95.62 & \scriptsize$\pm$0.06 & 2.65 & \scriptsize$\pm$0.13 \\
 & PRN & 87.24 & \scriptsize$\pm$0.44 & 92.14 & \scriptsize$\pm$0.23 & 4.89 & \scriptsize$\pm$0.29 & 94.70 & \scriptsize$\pm$0.20 & 98.07 & \scriptsize$\pm$0.09 & 3.37 & \scriptsize$\pm$0.15 \\
 & WRN & 90.09 & \scriptsize$\pm$0.41 & 94.72 & \scriptsize$\pm$0.20 & 4.63 & \scriptsize$\pm$0.27 & 95.40 & \scriptsize$\pm$0.13 & 98.90 & \scriptsize$\pm$0.07 & 3.50 & \scriptsize$\pm$0.17 \\
\midrule
\multirow{3}{*}{CIFAR-100} 
 & VGG-16 & 54.76 & \scriptsize$\pm$0.66 & 60.17 & \scriptsize$\pm$0.63 & 5.41 & \scriptsize$\pm$0.33 & 81.43 & \scriptsize$\pm$0.53 & 87.53 & \scriptsize$\pm$0.26 & 6.10 & \scriptsize$\pm$0.37 \\
 & PRN & 58.51 & \scriptsize$\pm$0.46 & 67.38 & \scriptsize$\pm$0.79 & 8.87 & \scriptsize$\pm$0.38 & 86.87 & \scriptsize$\pm$0.10 & 93.63 & \scriptsize$\pm$0.08 & 6.76 & \scriptsize$\pm$0.15 \\
 & WRN & 65.89 & \scriptsize$\pm$0.39 & 74.93 & \scriptsize$\pm$0.20 & 9.04 & \scriptsize$\pm$0.46 & 87.91 & \scriptsize$\pm$0.24 & 96.05 & \scriptsize$\pm$0.10 & 8.14 & \scriptsize$\pm$0.18 \\
\midrule
\multirow{2}{*}{ImageNet-100} 
 & ViG-Ti & 72.36 & \scriptsize$\pm$0.54 & 72.19 & \scriptsize$\pm$0.97 & -0.17 & \scriptsize$\pm$0.46 & 91.10 & \scriptsize$\pm$0.14 & 93.00 & \scriptsize$\pm$0.29 & 1.90 & \scriptsize$\pm$0.24 \\
 & ViT-S & 61.59 & \scriptsize$\pm$1.04 & 76.27 & \scriptsize$\pm$0.08 & 14.68 & \scriptsize$\pm$1.04 & 88.71 & \scriptsize$\pm$0.32 & 95.95 & \scriptsize$\pm$0.22 & 7.25 & \scriptsize$\pm$0.45 \\
\bottomrule
\end{tabular}
\caption{Accuracy and OSP results under $C_4$ rotation augmentation, showcasing the impact of stochastic weight averaging (SWA). Positive $\Delta$ values indicate an improvement after applying SWA. Model abbreviations: PRN (PreResNet-164), WRN (WideResNet-28$\times$10).}
\label{tab:swa_acc_osp_results}
\vspace{-1em}
\end{table*}

\textbf{Boost of accuracy and equivariance.} Table~\ref{tab:swa_acc_osp_results} summarizes the empirical impact of SWA on test accuracy and OSP across various datasets and architectures. A primary observation is the positive correlation between the gains in standard generalization ($\Delta$Acc) and the improvements in equivariance ($\Delta$OSP). Across different models, particularly on more complex tasks such as CIFAR-10 and CIFAR-100, boosts in test accuracy are consistently accompanied by proportionally large increases in OSP. For instance, the models that benefit the most in accuracy (e.g., PreResNet and WideResNet on CIFAR-100) concurrently achieve the highest absolute gains in OSP (up to 8.14\%). This strong correlation suggests that a significant portion of SWA's performance benefit stems directly from the network acquiring better transformation invariance.

Furthermore, the improvement in equivariance via SWA appears as a robust and universal phenomenon. The improvement in OSP is strictly positive across all tested configurations, even in regimes where standard accuracy gains are negative (the ViG-Tiny setting on ImageNet-100). This persistent enhancement in equivariance, irrespective of the fluctuations in overall task accuracy, agrees with our theoretical framework.

\begin{table*}[t]
\centering
\footnotesize
\setlength{\tabcolsep}{6pt}
\begin{tabular}{ll
    S[table-format=1.3]
    S[table-format=1.3]@{\,}l
    S[table-format=1.3]
    S[table-format=1.3]@{\,}l
    S[table-format=1.2]
}
\toprule
& & \multicolumn{3}{c}{Loss ($L$)} & \multicolumn{3}{c}{Equivariance Loss ($L_{\mathrm{eq}}$)} & {Ratio} \\
\cmidrule(lr){3-5} \cmidrule(lr){6-8}
Dataset & Model 
  & {Before} & \multicolumn{2}{c}{After} 
  & {Before} & \multicolumn{2}{c}{After} 
  & {$R_\mathrm{eq} / R$} \\
\midrule
\multirow{2}{*}{MNIST} 
 & MLP & 0.187 & 0.183 & \scriptsize$\pm$0.003 & 0.021 & 0.018 & \scriptsize$\pm$0.001 & 1.14 \\
 & CNN & 0.791 & 0.795 & \scriptsize$\pm$0.099 & 0.019 & 0.005 & \scriptsize$\pm$0.000 & 3.93 \\
\midrule
\multirow{2}{*}{FashionMNIST} 
 & MLP & 0.422 & 0.406 & \scriptsize$\pm$0.002 & 0.021 & 0.014 & \scriptsize$\pm$0.001 & 1.44 \\
 & CNN & 0.720 & 0.455 & \scriptsize$\pm$0.002 & 0.011 & 0.003 & \scriptsize$\pm$0.000 & 2.52 \\
\midrule
\multirow{3}{*}{CIFAR-10} 
 & VGG-16 & 0.430 & 0.311 & \scriptsize$\pm$0.004 & 0.125 & 0.041 & \scriptsize$\pm$0.001 & 2.21 \\
 & PRN & 0.499 & 0.300 & \scriptsize$\pm$0.009 & 0.079 & 0.017 & \scriptsize$\pm$0.001 & 2.79 \\
 & WRN$\times$10 & 0.407 & 0.220 & \scriptsize$\pm$0.009 & 0.083 & 0.009 & \scriptsize$\pm$0.001 & 4.98 \\
\midrule
\multirow{3}{*}{CIFAR-100} 
 & VGG-16 & 2.627 & 1.495 & \scriptsize$\pm$0.032 & 0.327 & 0.056 & \scriptsize$\pm$0.005 & 3.32 \\
 & PRN & 2.050 & 1.608 & \scriptsize$\pm$0.042 & 0.155 & 0.045 & \scriptsize$\pm$0.001 & 2.70 \\
 & WRN & 1.652 & 1.359 & \scriptsize$\pm$0.010 & 0.161 & 0.026 & \scriptsize$\pm$0.001 & 5.09 \\
\midrule
\multirow{2}{*}{ImageNet-100} 
 & ViG-T & 1.178 & 1.352 & \scriptsize$\pm$0.058 & 0.088 & 0.066 & \scriptsize$\pm$0.004 & 1.53 \\
 & ViT-S & 1.907 & 1.154 & \scriptsize$\pm$0.011 & 0.118 & 0.027 & \scriptsize$\pm$0.001 & 2.60 \\
\bottomrule
\end{tabular}
\caption{Loss and Equivariance loss ($L_{\mathrm{eq}}$) results under $C_4$ rotation augmentation. $R_{\mathrm{eq}} / R$ measures how much more efficiently SWA minimizes the equivariance loss compared to the overall cross-entropy loss. Model abbreviations: PRN (PreResNet-164), WRN (WideResNet-28$\times$10).}
\label{tab:swa_loss_leq_results}
\vspace{-1em}
\end{table*}
\textbf{Relative improvement ratio.}
In Section~\ref{sec:verification}, we empirically verified our main theorem (Theorem~\ref{thm:improvement-ratio-general}) on synthetic tasks tailored to our theoretical setting. On practical datasets, however, several assumptions underlying the theorem no longer hold (e.g., the unrestricted parameter space), and the architectures involve structural information — such as the exact layerwise irrep decomposition — that is intractable to compute. We therefore restrict ourselves to demonstrating the qualitative claim that SWA improves equivariance more than it improves predictive performance. We revisit the image classification tasks and compare the gains in loss and in equivariance loss, listed in Table~\ref{tab:swa_loss_leq_results}.

Across all architectures and datasets we observe $R_{\mathrm{eq}}/R > 1$, i.e.\ SWA reduces the equivariance loss $L_{\mathrm{eq}}$ strictly more efficiently than it reduces the overall loss $L$. The magnitudes admit a more quantitative reading. For the augmentation group $C_4$ used throughout Table~\ref{tab:swa_loss_leq_results}, the right-hand side of \eqref{eq:improv-rat-bound} is at most $r|G|/(r|G| - 1) = 4/3$, attained at $r = 1$ and in the limit of vanishing multiplicity factor $\max_\lambda m_\lambda/(d_\lambda n)$; larger $r$ or larger multiplicities only decrease it. 

The MLPs, which are the architectures closest to the setting of Theorem~\ref{thm:improvement-ratio-general} in that the parameter space is unrestricted ($\mathcal{L} = \mathcal{H}$), give ratios of $1.14$ (MNIST) and $1.44$ (FashionMNIST), i.e.\ they lie closely to the theoretical value $4/3$. In particular, the MNIST value falling slightly below $4/3$ is not a contradiction, since we disregard the structural facts that lead to actual values for the factors $r$ and $\max_\lambda m_\lambda/(d_\lambda n)$. In summary, this reads as evidence that the bound is close to tight in the regime where its assumptions hold.

The deep convolutional models behave differently, reaching ratios between $2.21$ and $5.09$, with bound $4/3$ being merely conservative for these settings. Specifically, the ratio shows a correlation with task complexity and model capacity. The deviation from perfect prediction is potentially the root cause of the looseness of the bound. Furthermore, the dimension count \eqref{eq:dimratio} assumes an unrestricted parameter space, whereas for residual networks, their admissible parameters form a proper affine subspace $\mathcal{L} \subset \mathcal{H}$ whose equivariant part $\mathcal{E}$ is a smaller fraction of $\mathcal{L}$ than the count suggests. Both effects push the true ratio upwards relative to the bound, and the trend in Table~\ref{tab:swa_loss_leq_results} broadly increases with model depth within each dataset, from VGG-16 to WideResNet-28$\times$10.
\begin{table}[t]
\centering
\footnotesize
\begin{tabular}{ll
    S[table-format=2.2]@{\,}l
    S[table-format=2.2]@{\,}l
    S[table-format=-1.2]@{\,}l
    S[table-format=2.2]@{\,}l
    S[table-format=2.2]@{\,}l
    S[table-format=-1.2]@{\,}l
}
\toprule
& & \multicolumn{6}{c}{Accuracy (\%)} & \multicolumn{6}{c}{OSP (\%)} \\
\cmidrule(lr){3-8} \cmidrule(lr){9-14}
Dataset & Group
  & \multicolumn{2}{c}{Before} & \multicolumn{2}{c}{After} & \multicolumn{2}{c}{$\Delta$}
  & \multicolumn{2}{c}{Before} & \multicolumn{2}{c}{After} & \multicolumn{2}{c}{$\Delta$} \\
\midrule
\multirow{3}{*}{DHFR}
 & Tetra & 79.90 & \scriptsize$\pm$2.07 & 81.45 & \scriptsize$\pm$1.57 & 1.55 & \scriptsize$\pm$1.82 & 96.36 & \scriptsize$\pm$0.18 & 98.13 & \scriptsize$\pm$1.15 & 1.77 & \scriptsize$\pm$1.13 \\
 & Octa  & 82.28 & \scriptsize$\pm$2.44 & 79.47 & \scriptsize$\pm$2.15 & -2.81 & \scriptsize$\pm$3.58 & 97.11 & \scriptsize$\pm$0.27 & 97.79 & \scriptsize$\pm$0.76 & 0.68 & \scriptsize$\pm$0.61 \\
 & Icosa & 80.23 & \scriptsize$\pm$2.44 & 78.42 & \scriptsize$\pm$4.09 & -1.81 & \scriptsize$\pm$2.23 & 96.96 & \scriptsize$\pm$0.40 & 99.42 & \scriptsize$\pm$0.51 & 2.46 & \scriptsize$\pm$0.47 \\
\cmidrule(l){2-14}
\multirow{3}{*}{BZR}
 & Tetra & 85.83 & \scriptsize$\pm$3.15 & 84.15 & \scriptsize$\pm$3.34 & -1.68 & \scriptsize$\pm$2.99 & 98.22 & \scriptsize$\pm$0.67 & 98.98 & \scriptsize$\pm$1.22 & 0.76 & \scriptsize$\pm$1.14 \\
 & Octa  & 83.13 & \scriptsize$\pm$2.51 & 85.61 & \scriptsize$\pm$3.04 & 2.48 & \scriptsize$\pm$1.07 & 97.68 & \scriptsize$\pm$0.57 & 99.29 & \scriptsize$\pm$0.71 & 1.61 & \scriptsize$\pm$0.49 \\
 & Icosa & 86.46 & \scriptsize$\pm$2.47 & 83.66 & \scriptsize$\pm$2.38 & -2.80 & \scriptsize$\pm$2.01 & 98.82 & \scriptsize$\pm$0.40 & 99.76 & \scriptsize$\pm$0.22 & 0.94 & \scriptsize$\pm$0.29 \\
\cmidrule(l){2-14}
\multirow{3}{*}{COX2}
 & Tetra & 78.54 & \scriptsize$\pm$2.05 & 77.87 & \scriptsize$\pm$0.89 & -0.67 & \scriptsize$\pm$1.27 & 96.75 & \scriptsize$\pm$0.45 & 98.62 & \scriptsize$\pm$0.95 & 1.87 & \scriptsize$\pm$0.98 \\
 & Octa  & 78.91 & \scriptsize$\pm$1.88 & 80.64 & \scriptsize$\pm$1.17 & 1.73 & \scriptsize$\pm$2.43 & 97.30 & \scriptsize$\pm$0.36 & 99.96 & \scriptsize$\pm$0.10 & 2.66 & \scriptsize$\pm$0.40 \\
 & Icosa & 78.68 & \scriptsize$\pm$3.13 & 80.00 & \scriptsize$\pm$1.17 & 1.32 & \scriptsize$\pm$3.99 & 97.66 & \scriptsize$\pm$0.18 & 99.65 & \scriptsize$\pm$0.45 & 2.00 & \scriptsize$\pm$0.52 \\
\bottomrule
\end{tabular}
\caption{Accuracy and OSP on the graph classification benchmarks under
augmentation by discrete rotation subgroups of $SO(3)$ (tetrahedral,
octahedral and icosahedral), before and after stochastic weight averaging.
Positive $\Delta$ values indicate improvement.}
\label{tab:results_graph}
\vspace{-1em}
\end{table}

\textbf{Generalization to graph classification tasks.}
The results above establish that SWA robustly improves equivariance on image classification tasks, occasionally independently of whether accuracy improves at all. However, it is natural to ask whether this pattern is specific to convolutional architectures and image data, or reflects a more general property of SWA. To test this, we repeat the same analysis on graph classification tasks, where both the underlying group action and the network architecture differ substantially from the image setting. 

Table~\ref{tab:results_graph} reports accuracy and OSP for the three
benchmarks under augmentation by the tetrahedral, octahedral and icosahedral
rotation groups. The pattern differs from the image experiments. The effect of weight averaging on the test accuracy is essentially random: five of the nine
configurations show a negative $\Delta\mathrm{Acc}$, the magnitudes range from
$-2.81\%$ to $+2.48\%$, and in every setting the error bar over seeds is
comparable to or larger than the mean. Equivariance, in contrast, improves in all
nine configurations, with $\Delta\mathrm{OSP}$ between $0.68\%$ and $2.66\%$ and
a clearly resolved effect in the majority of them.

\section{Conclusion}

We analyzed stochastic weight averaging (SWA) applied to networks trained on augmented data, framing the problem in terms of an Ornstein--Uhlenbeck approximation of the SGD trajectory near an equivariant minimizer. Decomposing the parameter space into equivariant and orthogonal subspaces $\E \oplus \E^{\perp}$, we showed that the loss ratios $R(T)$ and $R_{\mathcal{E}^{\perp}}(T)$ both diverge with the averaging time $T$, and that their relative rate is governed by the trace ratio $\mathrm{Tr}(H_{\mathcal{E}^{\perp}})/\mathrm{Tr}(H)$ at the minimizer. Associating the ratio with the neural tangent kernel in the infinite-width limit, we derived a bound (Theorem~\ref{thm:improvement-ratio-general}) showing that this relative ratio exceeds one under mild conditions on the group and the minimizer, establishing that SWA yields an equivariance boost more prominent than its performance boost. We verified this prediction on a synthetic teacher-student task where all structural constants are accessible exactly; we demonstrated the qualitative effect of SWA increasing equivariance across image and graph classification benchmarks spanning a variety of model families. 

Several future directions remain open. First, our analysis is confined to the lazy NTK regime. Extending it to feature learning, where the kernel evolves during training would bring the analysis closer to practical training dynamics.
Second, our analysis is local to a single equivariant minimizer. On the other hand, real SGD trajectories may transition between basins during training. Understanding how competition or transitions between minima interact with the equivariance boost is left to future work.
Furthermore, we treat the covariance matrix $\Sigma(\theta)$ as constant at the minimizer, whereas in practice it varies along the trajectory. Injecting this complexity into the approximation would enrich the framework and refine the theoretical predictions.
A complete empirical analysis of when SWA increases equivariance and when it does not remains an open problem. This involves factors including different tasks, architectures, and hyperparameters, therefore requiring more extensive studies that go beyond what was done in this work.


\section*{Acknowledgements}
This work was partially supported by the Wallenberg AI, Autonomous Systems and Software Program (WASP) funded by the Knut and Alice Wallenberg Foundation. The computations were enabled by resources provided by the National Academic Infrastructure for Supercomputing in Sweden (NAISS), partially funded by the Swedish Research Council through grant agreement no.\ 2025/22-1341. We thank Jakob Galley for interesting discussions.

\renewcommand*{\bibfont}{\normalfont\footnotesize}
\printbibliography

\appendix

\section{Proof of Proposition \ref{prop:ensemble}} \label{app:proof_ensemble}
    It is well known that the evolution of \eqref{eq:stochapp} is described by the Fokker-Planck equation: If $p_t$ is the density of the distribution at time $t$, we have
\begin{align} \label{eq:fokkerplanck}
    \partial_t p_t(\theta) =& -\nabla \cdot (p_t\nabla L ) + \nabla^2\cdot (Dp_t)
\end{align}    
Here, $D$ is the \emph{diffusion tensor}:
\begin{align}
    D(\theta) = \frac{1}{2}P(\theta)P(\theta)^T = \frac{1}{2}\Sigma(\theta)
\end{align}
We can express \eqref{eq:fokkerplanck} in operator form $\partial_t p_t = \Ac p_t$, with
\begin{equation}
    \Ac p = \nabla \cdot (p_t\nabla L ) + \nabla^2\cdot (Dp_t).
\end{equation}
It will be convenient to introduce also the push-forward operation as an operator:
\[
    (\Sg f)(\theta,t) = f(\bar g \theta,t).
\]
This for instance makes the difference between $\nabla L(\overline g \theta)$, which is $\Sg \nabla L$, and $\nabla[L(\bg \theta)]$, which is $\nabla(\Sg L)$.

The time-derivative $\partial_t$ trivially commutes with $\Sg$. Its commutation relations with the divergence and Laplace operators are as follows.
\begin{lemma} \label{lemma:diffops}
    For all $g\in G$, 
    \begin{align}
        \nabla \cdot[\bg^T \overline \Sg f]  =  \Sg [\nabla \cdot f],\quad\nabla^2 \cdot [\bg\Sg D \bg^T] = \Sg [\nabla^2 \cdot D]
    \end{align}
\end{lemma}
\begin{proof}
    Using Einstein summation, the chain rule gives, for any vector field \(h\),
    \begin{align}
        \nabla \cdot [\Sg h] = \partial_i \Sg h_i = g_{ji}\partial_j h_i 
    \end{align}
    where $g_{ij}$ denote the entries of the matrix corresponding to the action $\bg$. Since $(\bg^T  \Sg f)_i = g_{ki}\Sg f_k$, we get
    \begin{align}
        \nabla \cdot ( \bg^T \overline \Sg f) = g_{ki}\partial_i(\Sg f_k) = g_{ki}g_{ji} \Sg \partial_j f_k .
    \end{align}
    Now, since $\bg$ is orthogonal, we have $g_{ki}g_{ji} = g_{ki}g_{ij}^T=\delta_{kj}$, so 
    \begin{align}
         g_{ki}g_{ji} \Sg \partial_j f_k  = \delta_{kj} \Sg \partial_j f_k =  \Sg \partial_k f_k = \Sg [\nabla \cdot f],
    \end{align}
    which is the first claim. Similarly, the chain rule implies for every tensor field $E$
    \begin{align}
        \nabla^2 \cdot [\Sg E] = \partial_{i}\partial_j \Sg D_{ij} = g_{ki}g_{\ell j} \partial_k\partial_\ell D_{ij}.
    \end{align}
    Since $(\bg^TD \bg )_{ij} = g_{mi}D_{mn}g_{nj}$, we get
    \begin{align}
         \nabla^2 \cdot [\bg\Sg D \bg^T] &= g_{mi}g_{nj} \partial_{i}\partial_j[\Sg D_{mn}] = g_{mi}g_{nj}g_{ki}g_{\ell j} \Sg \partial_k\partial_\ell D_{mn} = \delta_{km}\delta_{n\ell}\Sg\partial_k\partial_\ell D_{mn} \\
         &= \Sg\partial_k\partial_\ell D_{k\ell} = \Sg [\nabla^2 \cdot D],
    \end{align}
    where we again applied $g_{ki}g_{ji} =\delta_{kj}$. 
\end{proof}

Using the above formulas, it is now not hard to prove that $\Ac$ commutes with $\Sg$.
\begin{lemma}
    We have $\Ac \Sg = \Sg \Ac$ for all $g\in G$.
\end{lemma}
\begin{proof}
   Lemma \ref{lemma:diffops} implies
    \begin{align}
        \Sg \Ac p  = \Sg [\nabla \cdot (p_t L)] + \Sg [\nabla^2\cdot (Dp_t)] = \nabla \cdot (\bg^T\Sg[p_t \nabla L]) + \nabla^2\cdot (\bg \Sg[Dp_t]\bg^T).
    \end{align}
    Now, we note that the chain rule and the invariance $\Sg L = L$ of the cumulative loss,  and Lemma \ref{lemma:cov} implies that
    \begin{align}
        \Sg \nabla L = \bg \nabla \Sg L = \bg \nabla L, \quad \Sg D = \bg^TD\bg,
    \end{align}
    so that 
    \begin{align}
        \nabla \cdot (\bg^T\Sg[p_t \nabla L]) + \nabla^2\cdot (\bg \Sg[Dp_t]\bg^T) &= \nabla \cdot ( \bg^T\bg \Sg p\nabla L) + \nabla^2\cdot (\Sg p_t \bg \bg^T  D\bg \bg^T) \\
        &=  \nabla \cdot (  \Sg p_t\nabla  L) + \nabla^2\cdot (\Sg p_t  D) = \Ac \Sg p_t.\qedhere
    \end{align}
\end{proof}

With the commutation $\Ac\Sg = \Sg \Ac$, it is not hard to prove Proposition \ref{prop:ensemble}.

\begin{proof}[Proof of Proposition \ref{prop:ensemble}]
    Since both $\Ac$ and $\partial_t$ commute with $\Sg$, so does the flow -- if $p_t$ solves the equation initialized at $p_0$, $\Sg p_t$ will do it for the one initialized at $\Sg p_0$. In particular, if we initialize at an invariant distribution, $\Sg p_0 =p_0$, we remain invariant, $\Sg p_t = p_t$, due to the existence and uniqueness of the solution. Consequently, $\theta \sim \Theta_t \Rightarrow \bg^{-1}\theta \sim \Theta_t$, therefore
    \begin{align}
        \overline{\N}_t (g_X x) = \mathbb{E}_{\theta \sim \Theta_t} (\N_\theta (g_Xx)) = \mathbb{E}_{\theta \sim \Theta_t} (g_Y\N_{\bg^{-1}\theta} (x)) \stackrel{\eqref{eq:parameter_transform}}{=}\mathbb{E}_{\theta \sim \Theta_t} (g_Y \N_\theta(x)) = g_Y \overline{\N}_t(x),
    \end{align}
    i.e., $\overline{\N}_t$ is equivariant.
\end{proof}

\end{document}